\documentclass{article}

\usepackage{microtype}
\usepackage{graphicx}
\usepackage{subcaption}
\usepackage{booktabs}
\usepackage{multirow}
\usepackage{xcolor}
\usepackage{colortbl}
\usepackage{enumitem}
\definecolor{bestbg}{HTML}{FFD1D1}    
\definecolor{secondbg}{HTML}{D1E8FF}  
\definecolor{rowhighlight}{gray}{0.96} 

\usepackage{hyperref}

\usepackage[preprint]{icml2026}

\usepackage{amsmath}
\usepackage{amssymb}
\usepackage{mathtools}
\usepackage{amsthm}

\usepackage[capitalize,noabbrev]{cleveref}

\theoremstyle{plain}
\newtheorem{theorem}{Theorem}[section]

\theoremstyle{definition}

\newtheorem{assumption}[theorem]{Assumption}
\theoremstyle{remark}

\usepackage{booktabs}
\usepackage{multirow}
\usepackage{adjustbox}

\usepackage{makecell}
\usepackage[table]{xcolor}
\usepackage{subcaption}

\usepackage{array}

\usepackage{newfloat}
\usepackage{listings}
\DeclareCaptionStyle{ruled}{labelfont=normalfont,labelsep=colon,strut=off} 
\floatstyle{ruled}
\newfloat{listing}{tb}{lst}{}
\floatname{listing}{Listing}

\usepackage{booktabs}

\usepackage[textsize=tiny]{todonotes}

\icmltitlerunning{Rethinking Irregular Time Series Forecasting  from the Perspective of Basis Functions}

\begin{document}

\twocolumn[
  \icmltitle{Rethinking Irregular Time Series Forecasting \\  from the Perspective of Basis Functions}



  \icmlsetsymbol{equal}{*}

  \begin{icmlauthorlist}
    \icmlauthor{Rongwen Li}{hnu}
    \icmlauthor{Changjian Chen}{hnu}
  \end{icmlauthorlist}

  \icmlaffiliation{hnu}{College of Computer Science and Electronic Engineering, Hunan University, Changsha, Hunan, China}

  \icmlcorrespondingauthor{Changjian Chen}{changjianchen@hnu.edu.cn}

  \icmlkeywords{Machine Learning, ICML}

  \vskip 0.3in
]



\printAffiliationsAndNotice{}  

\begin{abstract}
Irregular time series forecasting is crucial in many domains, such as healthcare and meteorological observation. 
However, due to the inherent characteristics of irregular time series, including sparse observations and non-uniform sampling, accurately predicting future dynamics remains challenging. 
In light of these two characteristics, many existing methods aggregate irregular observations into fixed-dimensional estimated response coefficients through predefined basis functions and use these coefficients as sequence representations. 
Nevertheless, this modeling paradigm still suffers from two key limitations: (i) a potential non-vanishing asymptotic bias caused by ignoring the sampling density of timestamps; and (ii) the limited adaptability of predefined basis functions to diverse temporal patterns.
In this study, we propose a Debiased Neural Basis-Function Network (DNBNet) to address these challenges.
Its core is a debiased neural basis-function response mechanism, which corrects asymptotic bias through importance sampling while parameterizing basis functions with neural networks to adapt to diverse temporal patterns. 
In addition, considering the sparsity of irregular data, we design a novel multi-scale decomposition module based on average pooling, together with a mass-aware fusion mechanism, to obtain richer representations. 
Finally, a dual-branch decoder is employed for forecasting. 
Extensive experiments on multiple real-world datasets demonstrate the effectiveness of DNBNet and its strong generalizability across diverse irregular time series scenarios. Our code can be obtained at https://github.com/hnu-vis/DNBNet.
\end{abstract}

\section{Introduction}
Irregular Multivariate Time Series (IMTS) consist of observations from multiple variables collected asynchronously at non-uniform, variable-specific timestamps, often with unequal sampling intervals and missing values.
Such data are widely encountered in various domains, including medical monitoring, sensor networks, and meteorological observations~\cite{DBLP:conf/aaai/GhassemiPNBCSF15,DBLP:journals/sensors/DecorteMLMLMV24,afrifa2020missing}. They are typically generated by underlying continuous-time dynamic processes, but can only be observed through sparse and irregular measurements in practice. As a result, learning reliable representations and making accurate predictions remain challenging.

\begin{table}[t]
\centering
\renewcommand{\arraystretch}{1}
\setlength{\tabcolsep}{3pt}
\caption{
Comparison of the types of predefined basis functions adopted by
different methods, including their specific mathematical forms and
function curves, reflecting their distinct modeling assumptions and
interpretations of the temporal dynamics in the data.}
\footnotesize
\begin{tabular}{@{}
>{\centering\arraybackslash}m{0.18\columnwidth}
>{\centering\arraybackslash}m{0.35\columnwidth}
>{\centering\arraybackslash}m{0.23\columnwidth}
>{\centering\arraybackslash}m{0.17\columnwidth}
@{}}
\toprule
\textbf{Basis Type}
&
\textbf{Formulation}
&
\textbf{Shape}
&
\textbf{Related Work}
\\
\midrule

\raisebox{-0.08in}{
  \makecell{
    RBF basis \\
    function
  }
}
&
\raisebox{-0.08in}{
  {\footnotesize
  $\displaystyle
  \exp\!\left(
    -\frac{(t-c_k)^2}{2\sigma_k^2}
  \right)
  $}
}
&
\raisebox{-0.5\height}{
  \includegraphics[
    height=0.5in,
    keepaspectratio
  ]{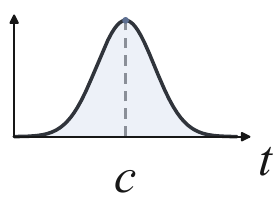}
}
&
\raisebox{-0.04in}{
  \parbox[c]{0.17\columnwidth}{
    \centering
    \cite{zhou2026revitalizing}
  }
}
\\

\midrule

\makecell{
  Fourier basis \\
  function
}
&
{\footnotesize
$\displaystyle
\begin{aligned}
  &\cos(2\pi\omega_k t),\\
  &\sin(2\pi\omega_k t)
\end{aligned}
$}
&
\raisebox{-0.5\height}{
  \includegraphics[
    height=0.5in,
    keepaspectratio
  ]{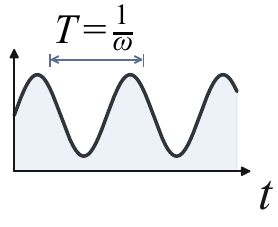}
}
&
\cite{DBLP:conf/icml/FonsSEFVV25,qiu2026bridging}
\\

\midrule

\makecell{
  Attention basis \\
  function
}
&
{\footnotesize
$\displaystyle
\frac{\exp(c_k Q^{\top}Kt)}{\sqrt{d}}
$}
&
\raisebox{-0.5\height}{
  \includegraphics[
    height=0.5in,
    keepaspectratio
  ]{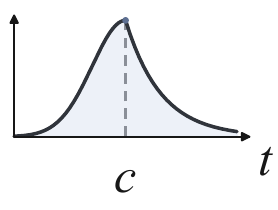}
}
&
\cite{DBLP:conf/iclr/ShuklaM21,lee2026adaptive}
\\

\midrule

\makecell{
  Soft-window \\
  basis function
}
&
{\scriptsize
$\displaystyle
\sigma\!\left(
  \frac{
    t_k^{\mathrm{right}}-t
  }{
    \operatorname{softplus}(\delta)
  }
\right)
\cdot
\sigma\!\left(
  \frac{
    t-t_k^{\mathrm{left}}
  }{
    \operatorname{softplus}(\delta)
  }
\right)
$}
&
\raisebox{-0.5\height}{
  \includegraphics[
    height=0.5in,
    keepaspectratio
  ]{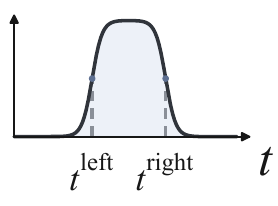}
}
&
\cite{liu2026rethinking}
\\

\bottomrule
\end{tabular}

\label{tab:basis_function}
\end{table}

To address this issue, many recent methods~\cite{DBLP:conf/iclr/ShuklaM21,zhou2026revitalizing,liu2026rethinking,lee2026adaptive,qiu2026bridging} attempt to map irregular observations into regular fixed-dimensional representations for subsequent prediction. Although these methods differ substantially in model design, we observe that most of them can be interpreted from a unified basis function perspective: \textbf{the model computes the responses of an irregular time series to a set of predefined temporal basis functions, and uses the resulting coefficients as the sequence representation.} 
These coefficients summarize how strongly the observed irregular time series aligns with each temporal basis, thereby encoding potential temporal patterns in a compact form.

From this perspective, the key distinction among different methods lies in the form of the adopted basis functions. As shown in Table~\ref{tab:basis_function}, existing methods use different predefined temporal basis functions and their corresponding response curves to encode temporal information. For example, radial basis functions (RBFs) typically emphasize local temporal responses and are therefore well suited for capturing information around specific time regions~\cite{zhou2026revitalizing}. In contrast, Fourier basis functions provide global temporal support and are more naturally aligned with modeling long-range structures such as periodicity and seasonality~\cite{qiu2026bridging}. Therefore, the choice of basis functions is not merely a technical detail; it implicitly determines which temporal patterns are easier for the model to represent, thereby introducing different structural priors and inductive biases.

This unified perspective also reveals two key limitations of existing methods.
\textbf{1) Potential asymptotic bias under irregular sampling:} existing discrete approximations for response coefficients are generally reasonable under regular or nearly uniform sampling, but they ignore the distribution of observation timestamps in irregular sampling scenarios, which may lead to a non-vanishing asymptotic bias.
\textbf{2) Limited adaptability of predefined bases:} although predefined basis functions introduce useful inductive biases, their fixed forms restrict the model's ability to adapt to different datasets and sampling patterns for irregular time series.

To overcome the above limitations, we propose a Debiased Neural Basis-Function Network (DNBNet). Specifically, we design a \textbf{Debiased Neural Basis-function Response} mechanism, which estimates the timestamp density via Kernel Density Estimation (KDE) to correct the asymptotic bias caused by non-uniform sampling.
Moreover, it parameterizes temporal basis functions with learnable neural networks, thereby improving the model's adaptability to complex temporal patterns. In addition, considering the sparsity of irregular time series data, we further develop a novel multi-scale decomposition based on average pooling and mass-aware fusion mechanism to obtain richer feature representations. Finally, the fused representations are fed into a dual-branch decoder for more accurate prediction. Extensive experiments on multiple real-world datasets demonstrate the effectiveness of the proposed framework.

\begin{itemize}
    \item \textbf{A unified basis-function perspective} of irregular time series forecasting, which reveals two key limitations of existing methods: potential asymptotic estimation bias caused by irregular sampling, and the limited ability of predefined basis functions to adaptively model diverse temporal patterns.

    \item \textbf{A DNBNet with a neural basis-response mechanism} that corrects the asymptotic bias and produces more expressive response coefficients. Building on this mechanism, we also incorporate multi-scale decomposition and mass-aware fusion to learn richer feature representations.

    \item \textbf{Extensive experiments} on multiple irregular time series forecasting benchmarks show that DNBNet consistently achieves superior predictive performance compared with existing competitive methods.
\end{itemize}




\section{Related Work}
\subsection{Irregular Multivariate Time Series Forecasting}
Recent IMTS forecasting methods can be roughly divided into three categories. \textbf{(i)} Some methods directly model the underlying continuous time based ODE or CDE~\cite{DBLP:conf/nips/MercataliFC24,oh2026flowpath}. For instance, Latent ODE~\cite{rubanova2019latent} leverages ODEs to construct continuous-time sequences in the latent space, thereby enabling interpolation and extrapolation. To avoid inefficient numerical integration for ODEs, Neural Flow~\cite{DBLP:conf/nips/BilosSRJG21} turn to directly model the solution of ODEs. \textbf{(ii)} Methods based on specific data structures model message passing among observation points~\cite{yalavarthi2024grafiti,DBLP:conf/iclr/ZhangZTZ22}. For example, SeFT~\cite{horn2020set} utilizes set, whereas GraFITi~\cite{yalavarthi2024grafiti} and HyperIMTS~\cite{DBLP:conf/icml/LiL0ZL025} use bipartite graphs and hypergraphs, respectively. \textbf{(iii)} Some methods learn compact representations through temporal alignment and aggregation~\cite{DBLP:conf/icml/ZhangYL0024,DBLP:conf/icml/LuoZ0025}. tPatchGNN~\cite{DBLP:conf/icml/ZhangYL0024} integrates region-level information by a fixed-span patching strategy, and APN~\cite{liu2026rethinking} designs an adaptive patch-span learning mechanism to avoid the meaningless and sparse feature representation from rigid patch size. Despite these advancements, many representation learning methods ignore the influence of timestamp density, thereby resulting in asymptotic biases.

\subsection{Basis Functions for Irregular Multivariate Time Series}
Basis functions provide a classical perspective for continuous signals, i.e., decomposing and characterizing complex signals through a set of response coefficients over temporal patterns. This idea has been widely used in regularly sampled time series analysis~\cite{DBLP:conf/iclr/OreshkinCCB20}. For example, some methods~\cite{DBLP:conf/icml/ZhouMWW0022,DBLP:conf/kdd/QiuW0GH025} capture periodic components of sequences via the Fast Fourier Transform (FFT). Similarly, many IMTS methods also utilize basis functions to extract temporal information. For instance, mTAN~\cite{DBLP:conf/iclr/ShuklaM21} proposes a time-aware attention mechanism to aggregate irregular observations using attention-based basis functions with predefined reference points. KAFNet~\cite{zhou2026revitalizing} utilizes Gaussian temporal kernel aggregation constructed from radial basis functions (RBFs) to efficiently model irregular time series. LSCD~\cite{DBLP:conf/icml/FonsSEFVV25} performs imputation with better consideration of the global structure of irregular time series through Fourier bases. However, previous methods usually ignore the inherent non-uniform sampling of irregular time series when estimating basis responses, thereby introducing asymptotic bias.

\section{Preliminaries and Motivation}
In this section, we first formalize the irregular time series forecasting problem and then present the discrete computation form of basis-function response coefficients adopted by existing methods. Based on this formulation, we further analyze the bias introduced by discrete response estimation under irregular sampling, which motivates us to design a debiased response estimation mechanism.

\subsection{Problem Definition}
Given an irregular multivariate time series
$\mathcal{O}=\{\mathcal{O}^n\}_{n=1}^{N}$ with $N$ variates, where
$\mathcal{O}^n=\{(t_i^n,x_i^n)\}_{i=1}^{L_n}$ denotes the observations of the $n$-th variate at non-uniform timestamps, and $(t_i^n,x_i^n)$ represents the value $x_i^n$ recorded at time $t_i^n$.
Irregular time series forecasting aims to predict future values at a set of query timestamps
$\mathcal{Q}=\{\mathcal{Q}^n\}_{n=1}^{N}$, with
$\mathcal{Q}^n=\{q_j^n\}_{j=1}^{Q_n}$. Formally, the goal is to learn a model $f_\theta$ that maps the historical observations and query timestamps to the corresponding predictions:
\begin{align}
    f_\theta(\mathcal{O}, \mathcal{Q})
    \rightarrow
    \hat{\mathcal{Y}} 
    =
    \left\{\{\hat{y}_j^n\}_{j=1}^{Q_n}\right\}_{n=1}^{N},
\end{align}
where $\hat{y}_j^n$ denotes the predicted value of the $n$-th variate at timestamp $q_j^n$.

\subsection{Basis-Function Response Representation}
Given a continuous-time signal $x(t)$ and a set of temporal basis functions $\{\phi_k(t)\}_{k=1}^{K}$, the ideal response coefficient with respect to the $k$-th basis is defined by the following integral:
\begin{align}
\label{eq1}
    c_k = \langle x,\psi_k\rangle =  \int_{\mathcal T} x(t)\psi_k(t)\,dt, 
\end{align}
where $\psi_k(t)\!\!=\!\!\frac{\phi_k(t)}{\int_{\mathcal T}\phi_k(t)\,dt}$ or $\psi_k(t)\!\!=\!\!\frac{\phi_k(t)}{\int_{\mathcal T}\phi_k^2(t)\,dt}$ correspond to the weighted average response and projection response for nonnegative and orthogonal basis functions, respectively. For example, a commonly used weighted average response coefficient is given by
$c_k\!\!=\!\!\frac{\int_{\mathcal T}x(t)\phi_k(t)\,dt}
{\int_{\mathcal T}\phi_k(t)\,dt}$.
In practice, only discrete observations $\{(t_i,x_i)\}_{i=1}^{L}$ are available. Existing methods therefore usually approximate the above continuous integral by summing over the observed data, e.g.,
$\hat{c}_k\!=\!\frac{\sum_{i=1}^{L}x_i\phi_k(t_i)}
{\sum_{i=1}^{L}\phi_k(t_i)}$~\cite{DBLP:conf/iclr/ShuklaM21,zhou2026revitalizing,liu2026rethinking}. In the following, we focus on the weighted average response formulation for clarity, while the analysis for the projection response is similar.

\subsection{Asymptotic Bias Analysis of Discrete Response Estimation}
\label{sec1}
Existing methods approximate the continuous-time response $c_k$ using the discrete computation form $\hat{c}_k$, which is reasonable under regular sampling. However, when the observations are irregularly sampled, this discrete approximation may lead to an asymptotic bias.
To facilitate the subsequent analysis, we consider a normalized time domain $\mathcal{T}=[0,1]$ and introduce the following mild assumption:
\begin{assumption}
\label{ass1}
    The continuous signal $x(t)$ and the  nonnegative basis functions $\phi_k(t)$ are bounded, i.e.,
    $|x(t)| \le 1$ and $|\phi_k(t)| \le m$ for any $t$ and $k$.
    Moreover, each basis function has a nontrivial mass under both the uniform time density and the timestamp sampling density:
    \begin{align}
        \min\left\{
        \int_{\mathcal{T}}\phi_k(t)dt,
        \int_{\mathcal{T}}\phi_k(t)p(t)dt
        \right\}
        \ge a.
    \end{align}
\end{assumption}

Assumption~\ref{ass1} is reasonable in practice. Data normalization and the construction of basis functions typically ensure the existence of upper bounds for $x(t)$ and $\phi_k(t)$, while the lower bound on the basis mass guarantees that each basis function carries sufficient information. Under Assumption~\ref{ass1}, we give the following theorem, which provides an estimation error bound between the discrete coefficient $\hat{c}_k$ and the ideal response coefficient $c_k$:

\begin{theorem}
    Given observations $\{(t_i,x_i)\}_{i=1}^{L}$, where $t_i \stackrel{\mathrm{i.i.d.}}{\sim} p(t)$,
    for any $\delta \in (0,1)$, when $L$ is sufficiently large,
    with probability at least $1-\delta$, the following inequality holds:
    \begin{align}
        \max_{k\in[K]} |\hat{c}_k-c_k|
        \le
        \frac{4m}{a}
        \sqrt{\frac{2\log(4K/\delta)}{L}}
        +
        \frac{2m}{a}\|p(t)-1\|_{L^1},\nonumber
    \end{align}
    where
    $\|p(t)-1\|_{L^1}=\int_{\mathcal{T}}|p(t)-1|dt$.
\end{theorem}

\begin{proof}
    See Appendix~A.
\end{proof}

\begin{figure*}
    \centering
    \includegraphics[width=1\linewidth]{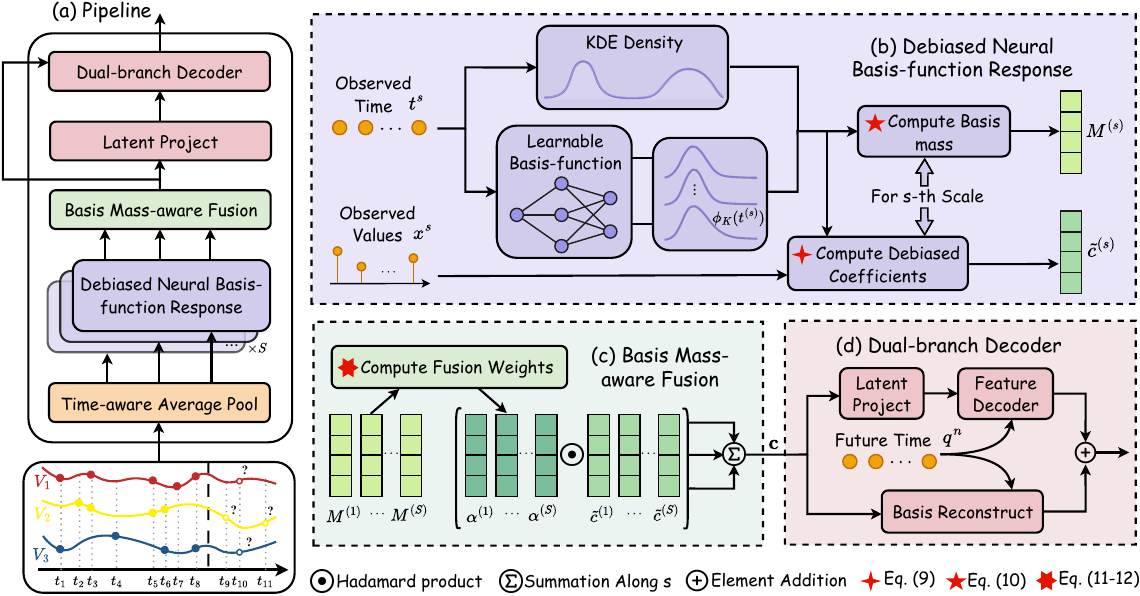}
    \caption{
    (a) The overall pipeline of DNBNet.
    (b) The debiased neural basis-function response computes debiased coefficients and basis masses using KDE and learnable basis functions.
    (c) The basis mass-aware fusion module aggregates coefficients from different scales.
    (d) The final prediction is obtained by combining the feature decoder with basis-function reconstruction.}
    \label{fig:pipeline}
\end{figure*}

The first term on the right-hand side of the above bound decreases as the number of observations $L$ increases, reflecting the standard finite-sample estimation error. In contrast, the second term is independent of $L$ and does not vanish with more observations unless the timestamp density $p(t)$ approaches the uniform distribution ($p(t)=1$), i.e., the observations are obtained through regular sampling. \textbf{Therefore, this term characterizes the potential asymptotic bias induced by irregular sampling, which does not vanish as L increases.} This observation motivates the need for a debiased response estimation mechanism that explicitly considers the sampling density of observed timestamps.

\section{Debiased Neural Basis-Function Network}
In this section, we introduce the proposed DNBNet in detail. As shown in Fig.~\ref{fig:pipeline}, we first generate observation subsequences at different scales through time-aware average pooling. We then design a debiased neural basis response mechanism to compute debiased and more expressive response coefficients. Next, we introduce a mass-aware fusion mechanism to integrate multi-scale response information. Finally, a dual decoder is employed to produce the final predictions.

\subsection{Time-aware Average Pool}
The underlying dynamics of irregular time series often exhibit heterogeneous temporal patterns across different time scales, while the inherent sparsity of the original sequence limits its ability to provide sufficient information for feature learning. This motivates us to construct subsequences at different scales to capture temporal patterns of varying granularities and obtain richer representations through fusion. Therefore, we first utilize time-aware average pooling to obtain the subsequence of the original input sequence $\mathcal{O}=\{(t_i,x_i)\}_{i=1}^{L}$ at the $s$-th scale:
\begin{align} 
    \mathcal{O}^{(s)}
    &=
    \text{AvgPool}\left(\mathcal{O} ; w_s,r_s\right), \\
    \mathcal{O}^{(s)}
    &=
    \{(t_i^{(s)},x_i^{(s)})\}_{i=1}^{L_s},
\end{align}
where $w_s$ and $r_s$ respectively denote the window size and stride defined by absolute time spans at scale $s$, and we let $\mathcal{O}^{(1)}$ denotes the sequence of original scale. 
Note that our model learns representation mainly in a channel-independent manner. For convenience, we temporarily omit the variate index $n$ in this section.

\subsection{Debiased Neural Basis-function Response}
According to the previous analysis, existing discrete computation methods suffer from asymptotic bias under irregular sampling. To correct this bias, we employ importance sampling to construct a density-corrected estimator  of the original integral:
\begin{equation}
\label{eq2}
        \scalebox{1}{$\displaystyle 
        \begin{aligned}
            \frac{\int_{\mathcal{T}} \frac{x(t)\phi_k(t)}{p(t)}p(t)dt}
             {\int_{\mathcal{T}} \frac{\phi_k(t)}{p(t)}p(t)dt} 
             \!=\! 
             \frac{\mathbb{E}_{t}\left[\frac{x(t)\phi_k(t)}{p(t)}\right]}
             {\mathbb{E}_{t}\left[\frac{\phi_k(t)}{p(t)}\right]} 
             \!\approx\! 
             \frac{\sum_{i=1}^{L}\frac{x_i\phi_k(t_i)}{p(t_i)}}
             {\sum_{i=1}^{L}\frac{\phi_k(t_i)}{p(t_i)}} 
             \!=\! \tilde c_k . 
        \end{aligned}
        $}
\end{equation}
The proposed estimator asymptotically corrects the bias. A detailed explanation of the debiasing property of our estimates is provided in the Appendix~A.2. 
Note that $\hat{c}_k$ of the existing methods is actually a special case of our formulation $\tilde{c}_k$ under the uniform timestamp distribution, i.e., $p(t)=1$. In irregular sampling scenarios, $1/p(t)$ can be interpreted as an importance weight that explicitly corrects the sampling bias: it suppresses the overemphasis on densely sampled regions while compensating for contributions from sparsely sampled regions. In our implementation, we use Kernel Density Estimation (KDE) to estimate the timestamp density:
\begin{align}
    p(t_i)=\frac{1}{L h}\sum_{j=1}^{L}\kappa\left(\frac{t_i-t_j}{h}\right),
\end{align}
where $\kappa(\cdot)$ denotes the kernel function, for which we adopt the Gaussian kernel. We further parameterize the bandwidth as a learnable variable, i.e., $h=\text{softplus}(\rho)$, to adapt to different temporal scales. 

In addition, as shown in Table~\ref{tab:basis_function}, many existing methods rely on predefined basis functions. The choice of basis functions usually reflects prior assumptions about the characteristics of the data. However, in the real scenarios, a fixed set of basis functions may be difficult to adapt to diverse datasets due to the changes of data distribution. 
Meanwhile, selecting appropriate basis functions for different datasets is highly dependent on domain expert knowledge, further limiting the generalization and applicability of the method.
Therefore, considering the strong approximation ability of neural networks, we directly learn basis functions themselves from raw data by fully connected neural networks. Specifically, let $\boldsymbol{\phi}(t) = \left[\phi_1(t), \ldots, \phi_K(t) \right] \in \mathbb{R}^K$ denote the vector formed by concatenating $K$ basis functions. We parameterize it with a two-layer MLP:
\begin{align}
\label{eq10}
    \boldsymbol{\phi}(t)=\text{softmax}(\text{MLP}_{\text{basis}}(t)),
\end{align}
where softmax is to bound basis functions, consistent with Assumption~\ref{ass1}. It also normalizes the responses across basis functions, thereby improving the stability of computation.

Next, we incorporate the learnable basis-function $\phi(t)$ into the debiased estimation in Eq.~\eqref{eq2}. For each scale-specific sequence $\mathcal{O}^{(s)}$, we compute the corresponding basis response coefficient with respect to the $k$-th basis function:
\begin{align}
\label{eq3}
    \tilde c_k^{(s)}
    =
    \frac{
    \sum_{i=1}^{L_s}
    x_{i}^{(s)}\phi_k(t_i^{(s)})/ p(t_i^{(s)})
    }{
    \sum_{i=1}^{L_s}
    \phi_k(t_i^{(s)})/ p(t_i^{(s)})
    },
\end{align}
In this way, each scale produces a set of basis response coefficients $\{\tilde{c}_k^{(s)}\}_{k=1}^{K}$.

\subsection{Basis Mass-aware Fusion}
To obtain richer representations, we further fuse the basis response coefficients across different scales. Note that the denominator in Eq.~\eqref{eq3} can be viewed as the basis mass of the $k$-th basis function at scale $s$:
\begin{align}
    M_k^{(s)}=\sum_{i=1}^{L_s} \phi_k(t_i^{(s)}) / p (t_i^{(s)}).
\end{align}
A larger $M_k^{(s)}$ indicates that the corresponding basis function receives stronger support at scale $s$, and thus its coefficient may contain more informative signals. Based on this observation, we design a mass-aware fusion mechanism to adaptively aggregate coefficients from different scales:
\begin{align}
    r_{s,k}
    &=
    \tau_{s,k}\log(1+M_k^{(s)})+\beta_{s,k}, \\
    \alpha_{s,k}
    &=
    \text{softmax} \left(\{ r_{s,k}\}_{s=1}^S\right), \\
    \tilde c_k
    &=
    \sum_{s=1}^{S} \alpha_{s,k} \tilde c_k^{(s)},
\label{eq4}
\end{align}
where $\tau_{s,k}$ and $\beta_{s,k}$ are learnable scale and bias parameters. The logarithmic transformation prevents excessively large basis masses from dominating the fusion weights.
Finally, we concatenate the fused coefficients of all basis functions as $\mathbf{c}=[\tilde c_1,\ldots, \tilde c_K] \in \mathbb{R}^{K}$, and further project it into a latent space to enhance its representation capacity:
\begin{align}
    \mathbf{h}
    &=
    \text{Linear}(\mathbf{c}), \\
    \mathbf{z}
    &=
    \text{LayerNorm}
    \left(
    \mathbf{h}
    +
    \text{MLP}(\mathbf{h})
    \right).
\end{align}
The representation $\mathbf{z} \in \mathbb{R}^d$ encodes debiased multi-scale basis responses for the subsequent forecasting decoder.

\subsection{Dual-Branch Forecasting Decoder}
To exploit both the latent representation and basis coefficients, we design a dual-branch forecasting decoder consisting of a feature decoder and a basis reconstruction branch.

Given a future query timestamp $q_l^n$ from the $n$-th variate at the $l$-th prediction step, following existing methods~\cite{DBLP:conf/icml/ZhangYL0024}, we construct its time embedding as follows:
\begin{align}
    \text{TE}(q_l^n)[d]
    =
    \begin{cases}
        \omega_0 \cdot q_l^n + b_0, & \text{if } d=0, \\
        \sin(\omega_d \cdot q_l^n + b_d), & \text{if } 0<d \leq D_t .
    \end{cases}
\end{align}
where $\{\omega_0,b_0\}$ and $\{\omega_d,b_d\}_{d=1}^{D_t}$ are learnable parameters for the linear and periodic components, respectively. To distinguish different variates, we add a variate embedding $\mathbf{e}_n$ to obtain $\mathbf{z}^n = \mathbf{z} + \mathbf{e}_n$. The feature-branch prediction is then computed by
\begin{align}
    \hat{y}_{l}^{n,\text{fea}}
    =
    \text{MLP}_{\text{fea}}
    \left(
    \left[
    \mathbf{z}^n,
    \text{TE}(q_l^n)
    \right]
    \right).
\end{align}

The basis functions provide a natural way to reconstruct the signal from the learned coefficients. This allows us to explicitly extrapolate future values by evaluating the learned basis functions at query timestamps and combining them with the basis coefficients. Therefore, in the basis branch, the prediction for the $n$-th variate at query timestamp $q_l^n$ is computed as
\begin{align}
    \hat{y}_{l}^{n,\text{basis}}
    =
    \sum_{k=1}^{K}
    \tilde c_k \phi_k(q_l^n),
\end{align}
where $\tilde c_k$ denotes the fused coefficient in Eq.~\eqref{eq4} and $\boldsymbol{\phi}(t)$ is the same learnable basis-function in Eq.~\eqref{eq10}. Finally, we fuse the two branches via a learnable gate:
\begin{align}
    \hat{y}_{l}^{n}
    =
    \lambda \, \hat{y}_{l}^{n,\text{fea}}
    +
    (1-\lambda) \, \hat{y}_{l}^{n,\text{basis}},
\end{align}
where $\lambda = \text{sigmoid}(\gamma)$ is a learnable parameter that balances the contribution of the feature and basis branch.

\section{Experiments}

\subsection{Experimental Setup}
\subsubsection{Datasets and Baselines}
To evaluate our method on irregular multivariate time series forecasting, we conduct experiments on five widely used benchmark datasets across multiple real-world domains. Specifically, we consider two healthcare datasets, PhysioNet and MIMIC, containing ICU clinical records; two biomechanics-related datasets, Human Activity and Student Life, consisting of sensor data from subjects performing various activities; and one climate dataset, USHCN, containing historical meteorological observations from U.S. weather stations. Following common practice, each dataset is split into training, validation, and test sets with proportions of $80\%$, $10\%$, and $10\%$, respectively.
We compare our method with 12 baseline models covering IMTS forecasting, classification, and interpolation, including PrimeNet~\cite{chowdhury2023primenet}, SeFT~\cite{horn2020set}, mTAN~\cite{DBLP:conf/iclr/ShuklaM21}, CRU~\cite{schirmer2022modeling}, GNeuralFlow~\cite{DBLP:conf/nips/MercataliFC24}, Raindrop~\cite{DBLP:conf/iclr/ZhangZTZ22}, tPatchGNN~\cite{DBLP:conf/icml/ZhangYL0024}, GraFITi~\cite{yalavarthi2024grafiti}, Warpformer~\cite{zhang2023warpformer}, Hi-Patch~\cite{DBLP:conf/icml/LuoZ0025}, KAFNet~\cite{zhou2026revitalizing}, and APN~\cite{liu2026rethinking}. Additional details on datasets and baselines are provided in Appendix~B.1.

\begin{table*}[t]
\centering
\setlength{\tabcolsep}{5pt}
\renewcommand{\arraystretch}{1.35}
\normalsize
\caption{MSE results of DNBNet and baselines on five irregular time-series datasets. Best results are in bold and second-best results are underlined.}
\begin{tabular}{lccccc}
\toprule
\textbf{Method} & \textbf{USHCN} & \textbf{Human Activity} & \textbf{Student Life} & \textbf{PhysioNet} & \textbf{MIMIC} \\
\midrule
PrimeNet & $0.7328\pm0.0000$ & $4.2564\pm0.0007$ & $0.8799\pm0.0001$ & $0.7952\pm0.0000$ & $0.9728\pm0.0001$ \\
SeFT & $0.6647\pm0.0008$ & $1.3750\pm0.0036$ & $0.8747\pm0.0026$ & $0.7801\pm0.0026$ & $0.9905\pm0.0003$ \\
mTAN & $0.4197\pm0.0161$ & $0.1017\pm0.0049$ & $0.6555\pm0.0064$ & $0.3787\pm0.0057$ & $0.9719\pm0.0169$ \\
CRU & $0.5179\pm0.0113$ & $0.1452\pm0.0009$ & $0.7487\pm0.0055$ & $0.6169\pm0.0036$ & $0.7370\pm0.0059$ \\
GNeuralFlow & $0.4971\pm0.0093$ & $0.2550\pm0.0373$ & $0.7168\pm0.0422$ & $0.3516\pm0.0028$ & $0.6631\pm0.0112$ \\
Raindrop & $0.4841\pm0.0064$ & $0.0939\pm0.0017$ & $0.6672\pm0.0054$ & $0.3491\pm0.0015$ & $0.6219\pm0.0040$ \\
tPatchGNN & $0.5499\pm0.1068$ & $0.0578\pm0.0007$ & $0.6332\pm0.0023$ & $0.3097\pm0.0014$ & $0.4673\pm0.0016$ \\
GraFITi & $0.4270\pm0.0093$ & $0.0606\pm0.0029$ & $0.6366\pm0.0001$ & $\mathbf{0.3021\pm0.0011}$ & $\underline{0.4349\pm0.0070}$ \\
Warpformer & $0.4263\pm0.0073$ & $0.0576\pm0.0009$ & $\underline{0.6267\pm0.0066}$ & $0.3153\pm0.0011$ & $0.4358\pm0.0018$ \\
Hi-Patch & $\underline{0.4108\pm0.0049}$ & $0.0592\pm0.0003$ & $0.6277\pm0.0010$ & $0.3226\pm0.0015$ & $0.4724\pm0.0067$ \\
KAFNet & $0.4132\pm0.0060$ & $0.0572\pm0.0002$ & $0.6312\pm0.0009$ & $0.3328\pm0.0015$ & $0.4612\pm0.0073$ \\
APN & $0.4262\pm0.0030$ & $\underline{0.0562\pm0.0004}$ & $0.6512\pm0.0015$ & $0.3133\pm0.0005$ & $0.4527\pm0.0093$ \\
\midrule
\rowcolor{gray!8}
DNBNet (Ours) & $\mathbf{0.4013\pm0.0108}$ & $\mathbf{0.0551\pm0.0001}$ & $\mathbf{0.6136\pm0.0002}$ & $\underline{0.3089\pm0.0005}$ & $\mathbf{0.4289\pm0.0012}$ \\
\bottomrule
\end{tabular}
\label{tab:mse_results}
\end{table*}

\begin{table}[t]
\centering
\setlength{\tabcolsep}{2pt}
\renewcommand{\arraystretch}{1.25}
\caption{Ablation study of DNBNet on three irregular time-series datasets, with MSE reported.}
\normalsize
\begin{adjustbox}{max width=\columnwidth}
\begin{tabular}{lccc}
\toprule
\textbf{Variant} & \textbf{USHCN} & \textbf{PhysioNet} & \textbf{MIMIC} \\
\midrule
w/o $p(t)$ & $0.4449\pm0.0291$ & $0.3161\pm0.0003$ & $0.4485\pm0.0012$ \\
Rep. RBF & $0.4144\pm0.0091$ & $0.3231\pm0.0007$ & $0.4507\pm0.0009$ \\
Rep. Fourier & $0.4101\pm0.0088$ & $0.3267\pm0.0002$ & $0.4869\pm0.0048$ \\
w/o AvgPool & $0.4195\pm0.0166$ & $0.3119\pm0.0006$ & $0.4494\pm0.0012$ \\
w/o BasDec & $0.4552\pm0.0424$ & $0.3092\pm0.0002$ & $0.4613\pm0.0023$ \\
\midrule
DNBNet (Full) & $\mathbf{0.4013\pm0.0108}$ & $\mathbf{0.3089\pm0.0005}$ & $\mathbf{0.4289\pm0.0012}$ \\
\bottomrule
\end{tabular}
\end{adjustbox}
\label{tab:balation}
\end{table}

\begin{figure}
    \centering
    \begin{subfigure}[b]{0.23\textwidth} 
        \centering
        \includegraphics[width=\textwidth]{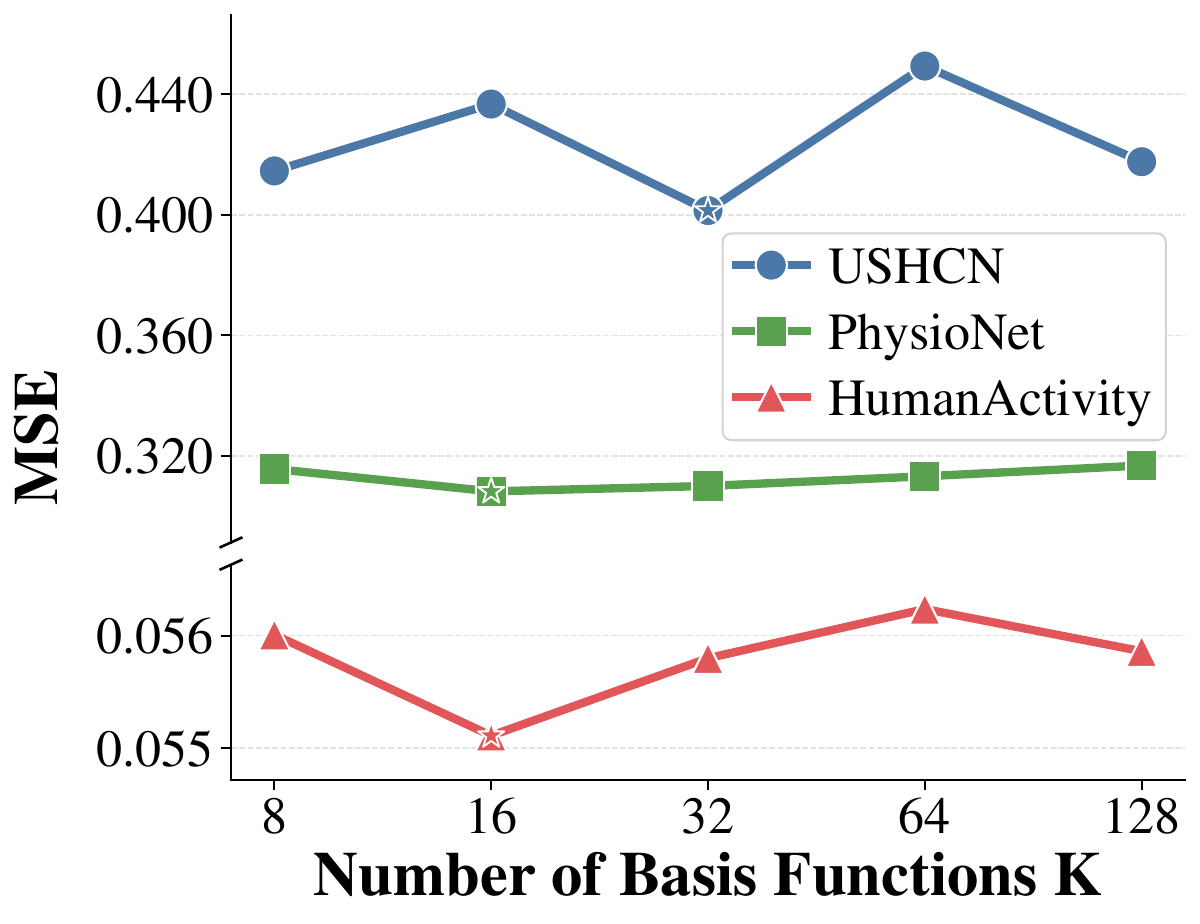} 
        \caption{MSE on three datasets w.r.t. different $K$.}
    \end{subfigure}
    \hfill 
    \begin{subfigure}[b]{0.23\textwidth}
        \centering
        \includegraphics[width=\textwidth]{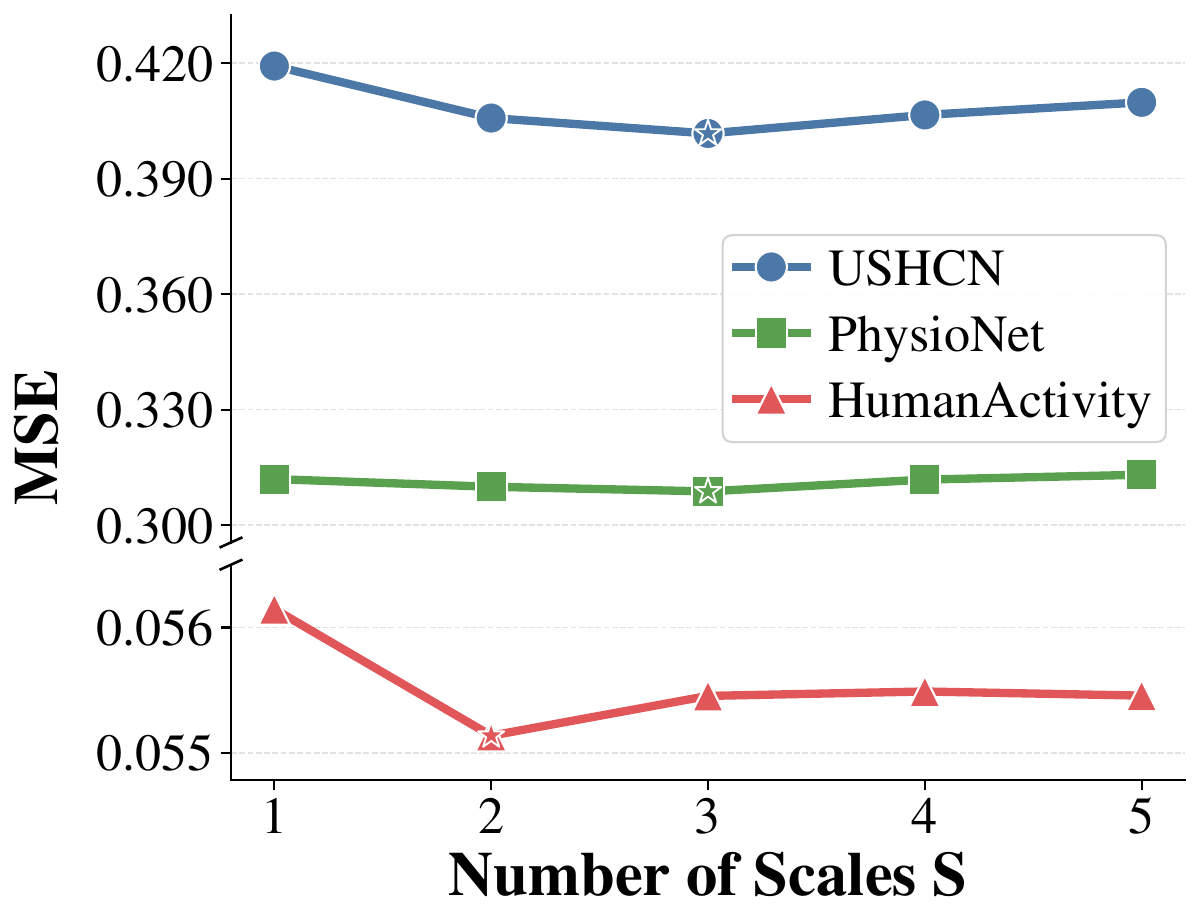}
        \caption{MSE on three datasets w.r.t. different $S$.}
    \end{subfigure}
    \caption{Parameter sensitivity study of number of basis functions $K$ and Scales $S$ on three datasets.}
    \label{fig:sensitive}
\end{figure}

\subsubsection{Implementation Details}
Following the prior work~\cite{DBLP:conf/nips/BilosSRJG21,DBLP:conf/icml/LuoZ0025}, for USHCN, we use 3 years of historical observations to predict the next 1 year; for Human Activity, the look-back window and prediction horizon are set to 3000 ms and 1000 ms, respectively; for Student Life, we use observations from the past 30 days to predict the next 30 days; and for PhysioNet and MIMIC, we use observations from the past 36 hours to predict the next 3 observation points. All experiments are implemented in PyTorch 2.5.1+cu121 and conducted on a single NVIDIA 4090 GPU. During training, all models are optimized using the AdamW optimizer, with mean squared error (MSE) adopted as the training objective. We set the maximum number of training epochs to 200 and use an early stopping patience of 10 to prevent overfitting. Mean squared error and mean absolute error (MAE) are used as evaluation metrics. Each experiment is repeated with five random seeds, and we report the mean and standard deviation of the results. 

\subsection{Main results}
Due to space limitations, we report the MSE results of our method and all baselines across all datasets in Table~\ref{tab:mse_results}, while the MAE results are provided in Appendix~B.2. As shown in the table, DNBNet achieves the best performance on four of the five datasets and obtains the second-best result on PhysioNet. In particular, compared with Warpformer, the average second-best baseline, DNBNet achieves an average relative MSE reduction of approximately $3.18\%$. Furthermore, compared with methods that also utilize basis-function based modeling, such as mTAN, KAFNet, and APN, DNBNet consistently achieves leading results on datasets from different domains, including healthcare and climate field. This indicates that, compared with directly relying on predefined basis functions, the proposed learnable basis functions can more effectively adapt to differences in sampling frequency or temporal irregularity across datasets, which is consistent with our motivation. In addition, we provide experimental results and analyses under different look-back windows and prediction horizons in the Appendix~B.3, further validating the stable advantages of DNBNet.

\subsection{Ablation Study}
We conduct ablation studies on three irregular time-series datasets to verify the effectiveness of our model components. The main results are reported in Table~\ref{tab:balation}, with complete results provided in the Appendix~B.4. Specifically, \textbf{(1)} w/o $p(t)$ removes the density factor in Eq.~\eqref{eq3}, leading to consistent performance degradation and confirming the effectiveness of our debiased response mechanism. \textbf{(2)} Rep. RBF and Rep. Fourier replace our learnable basis functions $\phi(t)$ with RBF and Fourier bases, respectively, both yielding inferior performance on all datasets. This suggests that our learnable basis-function is more robust and generalizable than predefined bases across different datasets. \textbf{(3)} w/o AvgPool and w/o BasDec remove the average pooling module and the basis-function prediction branch, respectively. The consistent performance degradation validates the effectiveness of leveraging multi-scale features and further demonstrates the potential of basis-function coefficients in representing future continuous dynamics. \textbf{(4)} To further validate the effectiveness of the debiased mechanism, we introduce the density debiasing factor $p(t)$ into two existing basis-function-based methods, KAFNet and APN. Specifically, we replace their biased weighted average responses with the estimator in Eq.~\eqref{eq2}. Table~\ref{tab:debiasing_factor} reports the original results of KAFNet and APN, as well as their performance after incorporating $p(t)$ on three datasets. Compared with the original biased weighted average aggregation, simply dividing by $p(t)$ yields consistent and stable improvements with almost no additional complexity. This verifies the effectiveness of our debiasing analysis.

\subsection{Parameter Sensitivity}
In this subsection, we study the impact of the number of basis functions and the multi-scale aggregation setting on model performance. As shown in Figure~\ref{fig:sensitive}, we report the results with different numbers of basis functions $K$ and scales $S$ on three datasets. We observe that using more basis functions does not necessarily lead to better performance; relatively small values, such as $K=16$ or $K=32$, are often sufficient to achieve the best or competitive results. Similarly, aggregating features from $2$ or $3$ scales already yields strong predictive performance, while further increasing the number of scales does not bring consistent gains. These results indicate that DNBNet is not sensitive to key hyperparameters and can maintain stable performance over a wide range of settings, providing practical guidance for transferring it to other datasets or application scenarios.

\subsection{Visualization and Efficiency Analysis}

To better understand the basis functions learned from data, we visualize the top four basis functions $\phi(t)$ with the largest mass on PhysioNet and MIMIC in Figure~\ref{fig:visualization}. For PhysioNet, the model learns localized basis functions similar to RBFs. Notably, each learned basis function has a distinct center, bandwidth, and peak location, suggesting that the model can adaptively adjust its local response range, which is difficult to achieve with predefined RBF bases. For the more challenging MIMIC dataset, DNBNet learns more complex and diverse basis-function shapes, including the soft window function $\phi_{12}$ and sigmoid-like functions $\phi_{0}$, $\phi_{8}$, and $\phi_{13}$. These basis functions capture response patterns over different temporal regions, further showing that DNBNet can learn temporal basis representations suited to the data distribution and dynamics, rather than being constrained by fixed predefined forms.
Figure~\ref{fig:efficiency} further compares the computational efficiency of DNBNet with six competitive baselines on MIMIC, the largest dataset, under the same batch size. DNBNet achieves the best forecasting performance while introducing nearly the lowest computational cost, except for APN. This efficiency mainly comes from avoiding quadratic-complexity operations and instead extracting temporal dynamics through learnable basis functions, resulting in a lightweight and efficient model design.

\begin{table}[t]
\centering
\setlength{\tabcolsep}{2.5pt}
\renewcommand{\arraystretch}{1.35}
\normalsize
\caption{Effectiveness of the density debiasing factor $1/p(t)$ on existing basis-function-based methods.}
\begin{adjustbox}{max width=\columnwidth}
\begin{tabular}{lccc}
\toprule
\textbf{Method} & \textbf{USHCN} & \textbf{PhysioNet} & \textbf{MIMIC} \\
\midrule
KAFNet w/o $p(t)$ & $0.4132 \pm 0.0060$ & $0.3328 \pm 0.0015$ & $0.4612 \pm 0.0073$ \\
KAFNet w/ $p(t)$  & $\mathbf{0.4050 \pm 0.0173}$ & $\mathbf{0.3117 \pm 0.0007}$ & $\mathbf{0.4488 \pm 0.0024}$ \\
\midrule
APN w/o $p(t)$    & $0.4262 \pm 0.0030$ & $0.3133 \pm 0.0005$ & $0.4527 \pm 0.0093$ \\
APN w/ $p(t)$     & $\mathbf{0.4193 \pm 0.0156}$ & $\mathbf{0.3094 \pm 0.0020}$ & $\mathbf{0.4409 \pm 0.0019}$ \\
\bottomrule
\end{tabular}
\end{adjustbox}
\label{tab:debiasing_factor}
\end{table}

\begin{figure}
    \centering
    \begin{subfigure}[b]{0.23\textwidth} 
        \centering
        \includegraphics[width=\textwidth]{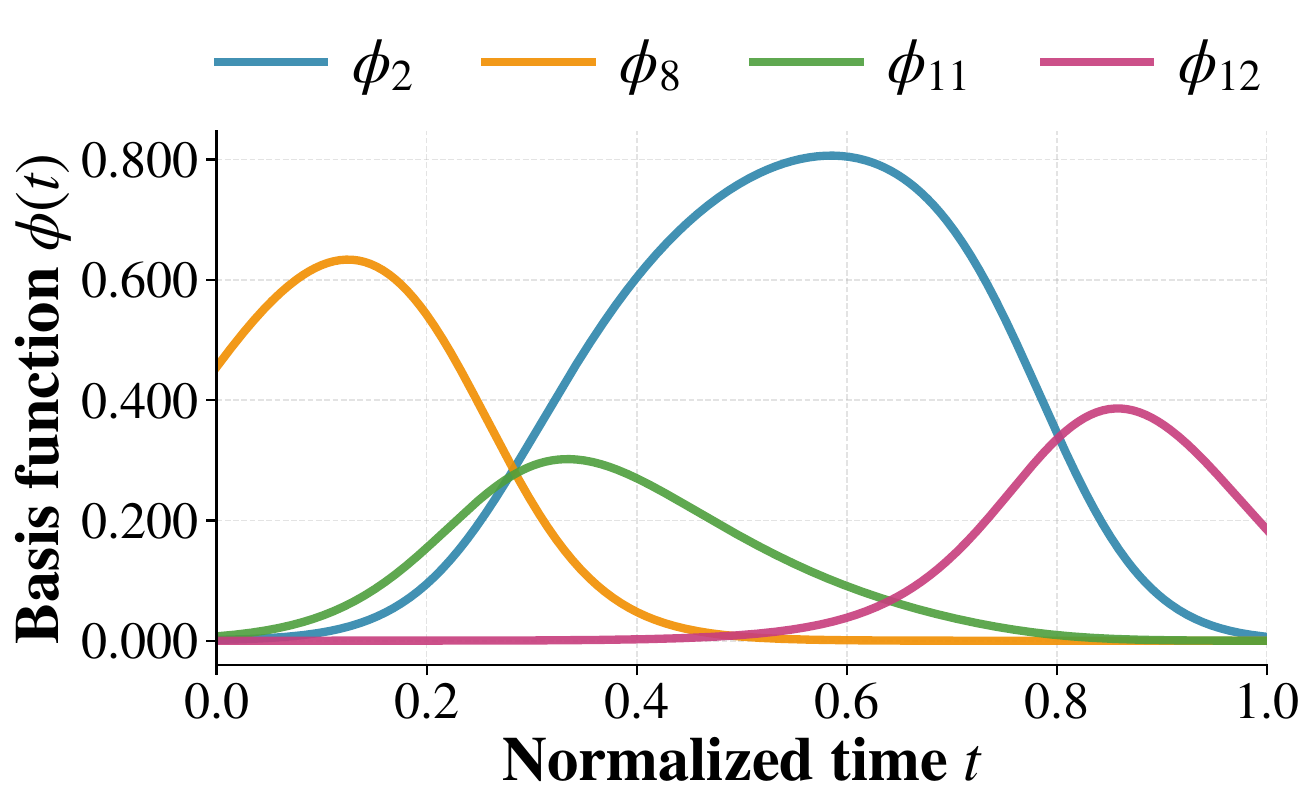} 
        \caption{The learned basis functions with the four largest basis function mass on PhysioNet dataset.}
    \end{subfigure}
    \hfill 
    \begin{subfigure}[b]{0.23\textwidth}
        \centering
        \includegraphics[width=\textwidth]{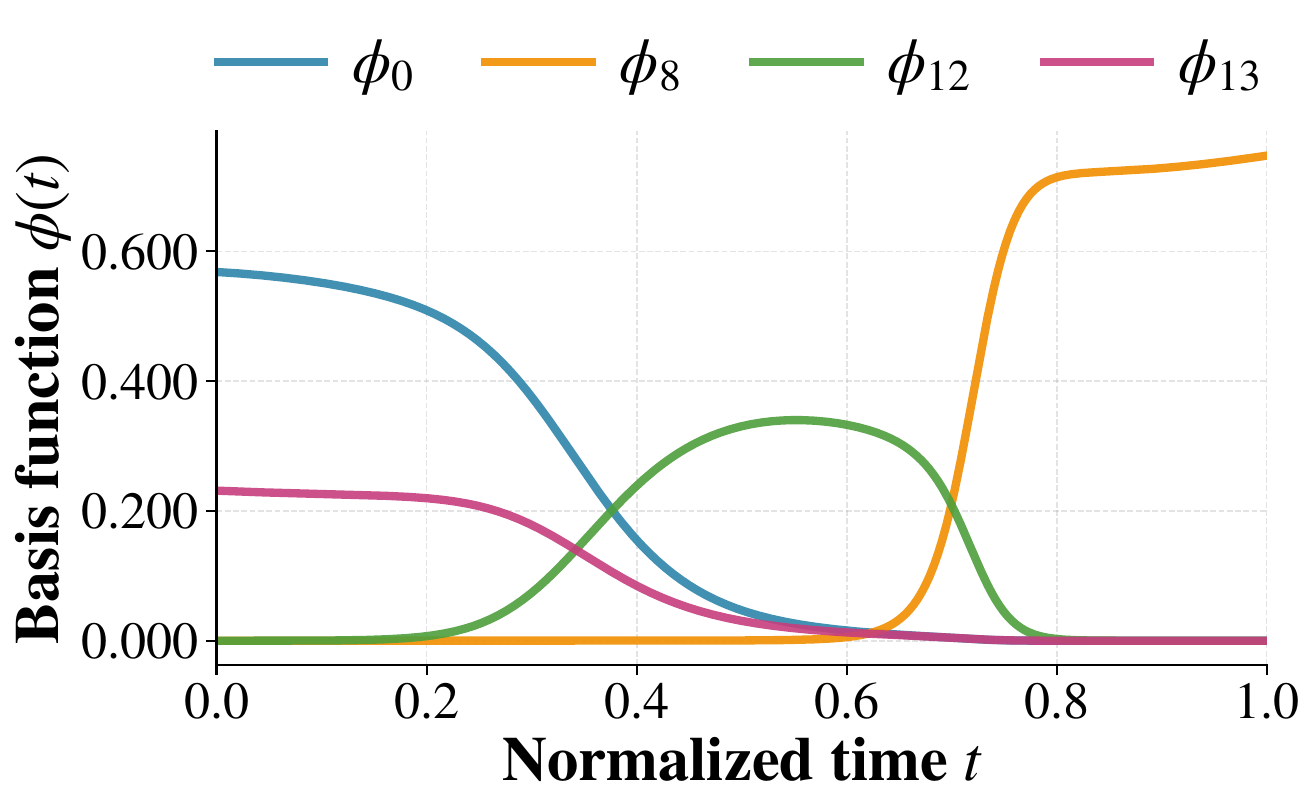}
        \caption{The learned basis functions with the four largest basis function mass on MIMIC dataset.}
    \end{subfigure}
    \caption{Visualization of the four learned basis functions with the largest basis function mass on two datasets.}
    \label{fig:visualization}
\end{figure}

\begin{figure}
    \centering
    \includegraphics[width=\linewidth]{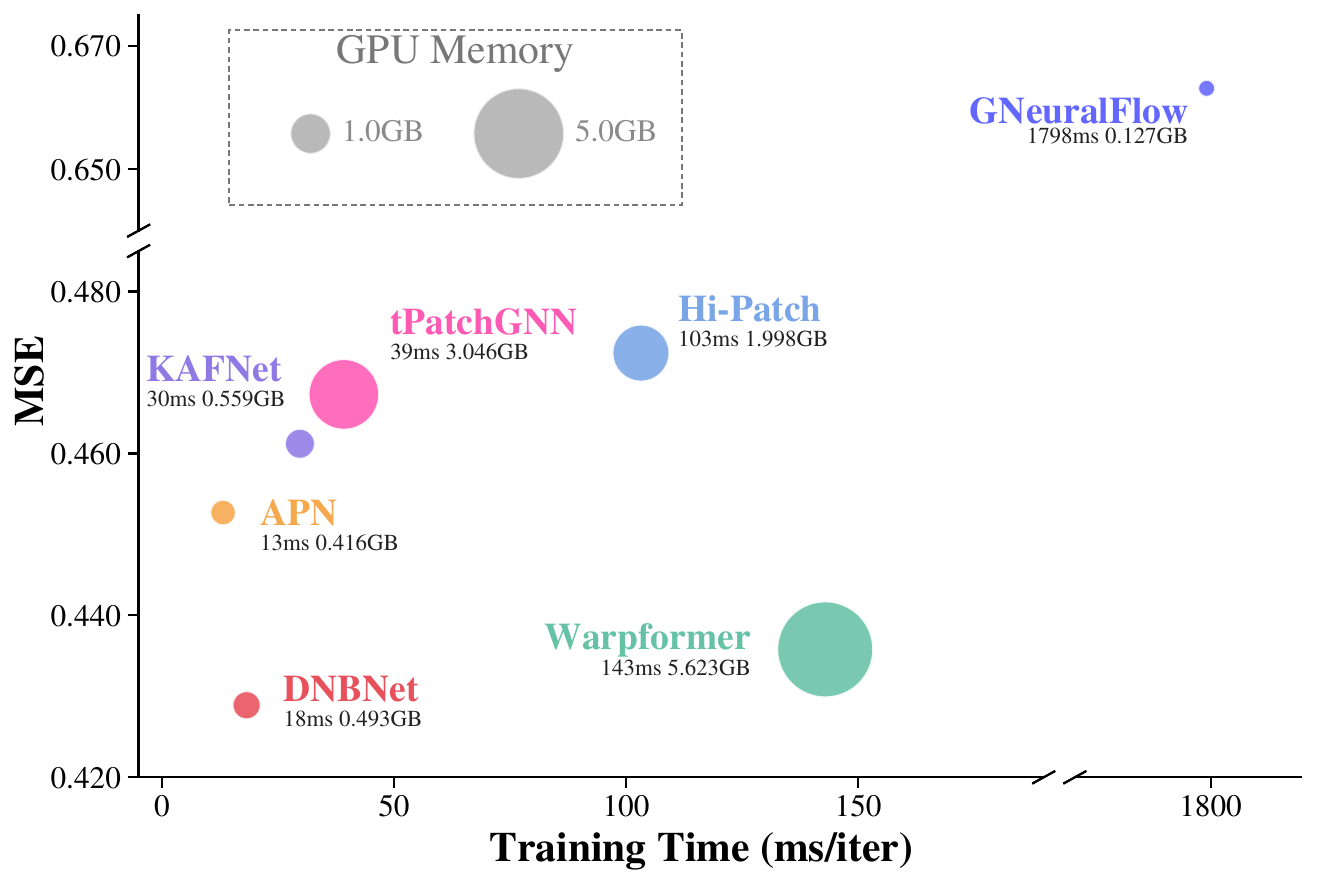}
    \caption{Model efficiency comparison with six competitive baselines on MIMIC dataset.}
    \label{fig:efficiency}
\end{figure}

\section{Conclusion}
We show that existing basis-based irregular time series forecasting methods can exhibit a non-vanishing asymptotic estimation bias when timestamp density is ignored, while their reliance on predefined basis functions may limit adaptability. To address these limitations, we propose the Debiased Neural Basis Function Network (DNBNet), a flexible framework that integrates density correction, learnable neural basis functions, multi-scale response extraction, and mass-aware fusion to enable robust and adaptive representation learning. DNBNet further employs a dual-branch decoder that combines implicit latent prediction with explicit basis-function reconstruction, thereby improving both predictive flexibility and reconstruction interpretability. Experiments on five real-world benchmarks demonstrate that DNBNet consistently achieves superior performance against strong baselines. 



\nocite{langley00}

\bibliography{main}

@inproceedings{DBLP:conf/aaai/GhassemiPNBCSF15,
  author       = {Marzyeh Ghassemi and
                  Marco A. F. Pimentel and
                  Tristan Naumann and
                  Thomas Brennan and
                  David A. Clifton and
                  Peter Szolovits and
                  Mengling Feng},
  title        = {A Multivariate Timeseries Modeling Approach to Severity of Illness
                  Assessment and Forecasting in {ICU} with Sparse, Heterogeneous Clinical
                  Data},
  booktitle    = {Proceedings of the Twenty-Ninth {AAAI} Conference on Artificial Intelligence,
                  January 25-30, 2015, Austin, Texas, {USA}},
  pages        = {446--453},
  year         = {2015},
}

@article{DBLP:journals/sensors/DecorteMLMLMV24,
  author       = {Thomas Decorte and
                  Steven Mortier and
                  Jonas J. Lembrechts and
                  Filip J. R. Meysman and
                  Steven Latr{\'{e}} and
                  Erik Mannens and
                  Tim Verdonck},
  title        = {Missing Value Imputation of Wireless Sensor Data for Environmental
                  Monitoring},
  journal      = {Sensors},
  volume       = {24},
  number       = {8},
  pages        = {2416},
  year         = {2024},
}

@article{afrifa2020missing,
  title={Missing data imputation of high-resolution temporal climate time series data},
  author={Afrifa-Yamoah, Eben and Mueller, Ute A and Taylor, Stephen M and Fisher, Aiden J},
  journal={Meteorological Applications},
  volume={27},
  number={1},
  pages={e1873},
  year={2020},
}

@inproceedings{DBLP:conf/iclr/ShuklaM21,
  author       = {Satya Narayan Shukla and
                  Benjamin M. Marlin},
  title        = {Multi-Time Attention Networks for Irregularly Sampled Time Series},
  booktitle    = {9th International Conference on Learning Representations, {ICLR} 2021,
                  Virtual Event, Austria, May 3-7, 2021},
  year         = {2021},
}

@inproceedings{zhou2026revitalizing,
  title={Revitalizing canonical pre-alignment for irregular multivariate time series forecasting},
  author={Zhou, Ziyu and Huang, Yiming and Wang, Yanyun and Wu, Yuankai and Kwok, James and Liang, Yuxuan},
  booktitle={Proceedings of the AAAI Conference on Artificial Intelligence},
  pages={29115--29123},
  year={2026}
}

@inproceedings{liu2026rethinking,
  title={Rethinking irregular time series forecasting: A simple yet effective baseline},
  author={Liu, Xvyuan and Qiu, Xiangfei and Wu, Xingjian and Li, Zhengyu and Guo, Chenjuan and Hu, Jilin and Yang, Bin},
  booktitle={Proceedings of the AAAI Conference on Artificial Intelligence},
  pages={23873--23881},
  year={2026}
}

@article{qiu2026bridging,
  title={Bridging Time and Frequency: A Joint Modeling Framework for Irregular Multivariate Time Series Forecasting},
  author={Qiu, Xiangfei and Yan, Kangjia and Liu, Xvyuan and Wu, Xingjian and Hu, Jilin},
  journal={arXiv preprint arXiv:2602.00582},
  year={2026}
}

@article{rubanova2019latent,
  title={Latent ordinary differential equations for irregularly-sampled time series},
  author={Rubanova, Yulia and Chen, Ricky TQ and Duvenaud, David K},
  journal={Advances in neural information processing systems},
  volume={32},
  year={2019}
}

@inproceedings{DBLP:conf/nips/BilosSRJG21,
  author       = {Marin Bilos and
                  Johanna Sommer and
                  Syama Sundar Rangapuram and
                  Tim Januschowski and
                  Stephan G{\"{u}}nnemann},
  title        = {Neural Flows: Efficient Alternative to Neural ODEs},
  booktitle    = {Advances in Neural Information Processing Systems 34: Annual Conference
                  on Neural Information Processing Systems 2021, NeurIPS 2021, December
                  6-14, 2021, virtual},
  pages        = {21325--21337},
  year         = {2021},
}

@inproceedings{horn2020set,
  title={Set functions for time series},
  author={Horn, Max and Moor, Michael and Bock, Christian and Rieck, Bastian and Borgwardt, Karsten},
  booktitle={International Conference on Machine Learning},
  pages={4353--4363},
  year={2020},
}

@inproceedings{yalavarthi2024grafiti,
  title={Grafiti: Graphs for forecasting irregularly sampled time series},
  author={Yalavarthi, Vijaya Krishna and Madhusudhanan, Kiran and Scholz, Randolf and Ahmed, Nourhan and Burchert, Johannes and Jawed, Shayan and Born, Stefan and Schmidt-Thieme, Lars},
  booktitle={Proceedings of the AAAI Conference on Artificial Intelligence},
  pages={16255--16263},
  year={2024}
}

@inproceedings{DBLP:conf/icml/LiL0ZL025,
  author       = {Boyuan Li and
                  Yicheng Luo and
                  Zhen Liu and
                  Junhao Zheng and
                  Jianming Lv and
                  Qianli Ma},
  title        = {HyperIMTS: Hypergraph Neural Network for Irregular Multivariate Time
                  Series Forecasting},
  booktitle    = {Forty-second International Conference on Machine Learning, {ICML}
                  2025, Vancouver, BC, Canada, July 13-19, 2025},
  year         = {2025},
}

@inproceedings{DBLP:conf/icml/ZhangYL0024,
  author       = {Weijia Zhang and
                  Chenlong Yin and
                  Hao Liu and
                  Xiaofang Zhou and
                  Hui Xiong},
  editor       = {Ruslan Salakhutdinov and
                  Zico Kolter and
                  Katherine A. Heller and
                  Adrian Weller and
                  Nuria Oliver and
                  Jonathan Scarlett and
                  Felix Berkenkamp},
  title        = {Irregular Multivariate Time Series Forecasting: {A} Transformable
                  Patching Graph Neural Networks Approach},
  booktitle    = {Forty-first International Conference on Machine Learning, {ICML} 2024,
                  Vienna, Austria, July 21-27, 2024},
  pages        = {60179--60196},
  year         = {2024},
}

@inproceedings{DBLP:conf/nips/MercataliFC24,
  author       = {Giangiacomo Mercatali and
                  Andr{\'{e}} Freitas and
                  Jie Chen},
  title        = {Graph Neural Flows for Unveiling Systemic Interactions Among Irregularly
                  Sampled Time Series},
  booktitle    = {Advances in Neural Information Processing Systems 37: Annual Conference
                  on Neural Information Processing Systems 2024, NeurIPS 2024, Vancouver,
                  BC, Canada, December 10 - 15, 2024},
  year         = {2024},
}

@inproceedings{oh2026flowpath,
  title={FlowPath: Learning Data-Driven Manifolds with Invertible Flows for Robust Irregularly-sampled Time Series Classification},
  author={Oh, YongKyung and Lim, Dong-Young and Kim, Sungil},
  booktitle={Proceedings of the AAAI Conference on Artificial Intelligence},
  pages={24594--24603},
  year={2026}
}

@inproceedings{DBLP:conf/iclr/ZhangZTZ22,
  author       = {Xiang Zhang and
                  Marko Zeman and
                  Theodoros Tsiligkaridis and
                  Marinka Zitnik},
  title        = {Graph-Guided Network for Irregularly Sampled Multivariate Time Series},
  booktitle    = {The Tenth International Conference on Learning Representations, {ICLR}
                  2022, Virtual Event, April 25-29, 2022},
  year         = {2022},
}

@inproceedings{DBLP:conf/icml/LuoZ0025,
  author       = {Yicheng Luo and
                  Bowen Zhang and
                  Zhen Liu and
                  Qianli Ma},
  title        = {Hi-Patch: Hierarchical Patch {GNN} for Irregular Multivariate Time
                  Series},
  booktitle    = {Forty-second International Conference on Machine Learning, {ICML}
                  2025, Vancouver, BC, Canada, July 13-19, 2025},
  year         = {2025},
}

@inproceedings{DBLP:conf/iclr/OreshkinCCB20,
  author       = {Boris N. Oreshkin and
                  Dmitri Carpov and
                  Nicolas Chapados and
                  Yoshua Bengio},
  title        = {{N-BEATS:} Neural basis expansion analysis for interpretable time
                  series forecasting},
  booktitle    = {8th International Conference on Learning Representations, {ICLR} 2020,
                  Addis Ababa, Ethiopia, April 26-30, 2020},
  year         = {2020},
}

@inproceedings{DBLP:conf/icml/ZhouMWW0022,
  author       = {Tian Zhou and
                  Ziqing Ma and
                  Qingsong Wen and
                  Xue Wang and
                  Liang Sun and
                  Rong Jin},
  title        = {FEDformer: Frequency Enhanced Decomposed Transformer for Long-term
                  Series Forecasting},
  booktitle    = {International Conference on Machine Learning, {ICML} 2022, 17-23 July
                  2022, Baltimore, Maryland, {USA}},
  pages        = {27268--27286},
  year         = {2022},
}

@inproceedings{DBLP:conf/kdd/QiuW0GH025,
  author       = {Xiangfei Qiu and
                  Xingjian Wu and
                  Yan Lin and
                  Chenjuan Guo and
                  Jilin Hu and
                  Bin Yang},
  title        = {{DUET:} Dual Clustering Enhanced Multivariate Time Series Forecasting},
  booktitle    = {Proceedings of the 31st {ACM} {SIGKDD} Conference on Knowledge Discovery
                  and Data Mining, V.1, {KDD} 2025, Toronto, ON, Canada, August 3-7,
                  2025},
  pages        = {1185--1196},
  year         = {2025},
}

@inproceedings{chowdhury2023primenet,
  title={Primenet: Pre-training for irregular multivariate time series},
  author={Chowdhury, Ranak Roy and Li, Jiacheng and Zhang, Xiyuan and Hong, Dezhi and Gupta, Rajesh K and Shang, Jingbo},
  booktitle={Proceedings of the AAAI Conference on Artificial Intelligence},
  pages={7184--7192},
  year={2023}
}

@inproceedings{schirmer2022modeling,
  title={Modeling irregular time series with continuous recurrent units},
  author={Schirmer, Mona and Eltayeb, Mazin and Lessmann, Stefan and Rudolph, Maja},
  booktitle={International conference on machine learning},
  pages={19388--19405},
  year={2022},
}

@inproceedings{zhang2023warpformer,
  title={Warpformer: A multi-scale modeling approach for irregular clinical time series},
  author={Zhang, Jiawen and Zheng, Shun and Cao, Wei and Bian, Jiang and Li, Jia},
  booktitle={Proceedings of the 29th ACM SIGKDD Conference on Knowledge Discovery and Data Mining},
  pages={3273--3285},
  year={2023}
}

@article{lee2026adaptive,
  title={Adaptive Time Encoding for Irregular Multivariate Time-Series Classification},
  author={Lee, Sangho and Min, Kyeongseo and Son, Youngdoo and Do, Hyungrok},
  journal={Advances in neural information processing systems},
  volume={38},
  pages={101795--101824},
  year={2026}
}

@article{johnson2016mimic,
  title={MIMIC-III, a freely accessible critical care database},
  author={Johnson, Alistair EW and Pollard, Tom J and Shen, Lu and Lehman, Li-wei H and Feng, Mengling and Ghassemi, Mohammad and Moody, Benjamin and Szolovits, Peter and Anthony Celi, Leo and Mark, Roger G},
  journal={Scientific data},
  volume={3},
  number={1},
  pages={1--9},
  year={2016},
  publisher={Nature Publishing Group}
}

@inproceedings{silva2012predicting,
  title={Predicting in-hospital mortality of icu patients: The physionet/computing in cardiology challenge 2012},
  author={Silva, Ikaro and Moody, George and Scott, Daniel J and Celi, Leo A and Mark, Roger G},
  booktitle={2012 computing in cardiology},
  pages={245--248},
  year={2012},
  organization={IEEE}
}

@techreport{menne2015long,
  title={Long-term daily and monthly climate records from stations across the contiguous United States (US Historical Climatology Network)},
  author={Menne, Matthew J and Williams Jr, Claude N and Vose, Russell S},
  year={2015},
  institution={Environmental System Science Data Infrastructure for a Virtual Ecosystem; CDIAC}
}

@article{nepal2024capturing,
  title={Capturing the college experience: a four-year mobile sensing study of mental health, resilience and behavior of college students during the pandemic},
  author={Nepal, Subigya and Liu, Wenjun and Pillai, Arvind and Wang, Weichen and Vojdanovski, Vlado and Huckins, Jeremy F and Rogers, Courtney and Meyer, Meghan L and Campbell, Andrew T},
  journal={Proceedings of the ACM on interactive, mobile, wearable and ubiquitous technologies},
  volume={8},
  number={1},
  pages={1--37},
  year={2024},
  publisher={ACM New York, NY, USA}
}

@inproceedings{DBLP:conf/icml/FonsSEFVV25,
  author       = {Elizabeth Fons and
                  Alejandro Sztrajman and
                  Yousef El{-}Laham and
                  Luciana Ferrer and
                  Svitlana Vyetrenko and
                  Manuela Veloso},
  title        = {{LSCD:} Lomb-Scargle Conditioned Diffusion for Time series Imputation},
  booktitle    = {Forty-second International Conference on Machine Learning, {ICML}
                  2025, Vancouver, BC, Canada, July 13-19, 2025},
  volume       = {267},
  year         = {2025},
}
\bibliographystyle{icml2026}

\newpage
\appendix








\section{Proof of Theorem~1}
For each basis function $\phi_k$, the ideal continuous-time response coefficient is defined as
\begin{align}
    c_k =
    \frac{
    \int_{\mathcal{T}} x(t)\phi_k(t)dt
    }{
    \int_{\mathcal{T}} \phi_k(t)dt
    }.
\end{align}
The uncorrected discrete estimator is defined as
\begin{align}
    \hat{c}_k =
    \frac{
    \sum_{i=1}^{L} x_i\phi_k(t_i)
    }{
    \sum_{i=1}^{L} \phi_k(t_i)
    }.
\end{align}
Since the timestamps are sampled from density $p(t)$, the discrete estimator does not directly approximate $c_k$. Instead, it approximates the response coefficient under the timestamp sampling density:
\begin{align}
    c_k^p =
    \frac{
    \int_{\mathcal{T}} x(t)\phi_k(t)p(t)dt
    }{
    \int_{\mathcal{T}} \phi_k(t)p(t)dt
    }.
\end{align}
Therefore, we decompose the error as
\begin{align}
    |\hat{c}_k-c_k|
    \le
    |\hat{c}_k-c_k^p|
    +
    |c_k^p-c_k|.
\end{align}
The first term corresponds to the finite-sample approximation error, while the second term corresponds to the sampling bias induced by the non-uniform timestamp density.

We first bound the finite-sample error. Define
\begin{align}
    Y_i^k=x(t_i)\phi_k(t_i),
    \qquad
    Z_i^k=\phi_k(t_i).
\end{align}
Then
\begin{align}
    \hat{c}_k=
    \frac{\bar{Y}_k}{\bar{Z}_k},
    \qquad
    c_k^p=
    \frac{\mathbb{E}_{p}[Y_i^k]}{\mathbb{E}_{p}[Z_i^k]},
\end{align}
where
\begin{align}
    \bar{Y}_k=\frac{1}{L}\sum_{i=1}^{L}Y_i^k,
    \qquad
    \bar{Z}_k=\frac{1}{L}\sum_{i=1}^{L}Z_i^k.
\end{align}
Since $|x(t)|\le 1$ and $0\le \phi_k(t)\le m$, we have
\begin{align}
    |Y_i^k|\le m,
    \qquad
    0\le Z_i^k\le m.
\end{align}
By Hoeffding's inequality and a union bound over all $k\in[K]$ for both
$\bar{Y}_k$ and $\bar{Z}_k$, with probability at least $1-\delta$, the following inequalities hold simultaneously:
\begin{align}
    \max_{k\in[K]}
    |\bar{Y}_k-\mathbb{E}_{p}[Y_i^k]|
    \le
    \epsilon_L,
    \qquad
    \max_{k\in[K]}
    |\bar{Z}_k-\mathbb{E}_{p}[Z_i^k]|
    \le
    \epsilon_L,
\end{align}
where
\begin{align}
    \epsilon_L
    =
    m\sqrt{\frac{2\log(4K/\delta)}{L}}.
\end{align}
When $L$ is sufficiently large such that $\epsilon_L\le a/2$, i.e., $L\ge\frac{8m^2}{a^2}\log\frac{4K}{\delta}$,
we have
\begin{align}
    \bar{Z}_k
    \ge
    \mathbb{E}_{p}[Z_i^k]-\epsilon_L
    \ge
    \frac{a}{2}.
\end{align}
Now we compare the two ratios:
\begin{align}
    |\hat{c}_k-c_k^p|
    &=
    \left|
    \frac{\bar{Y}_k}{\bar{Z}_k}
    -
    \frac{\mathbb{E}_{p}[Y_i^k]}{\mathbb{E}_{p}[Z_i^k]}
    \right| \nonumber \\
    &\le
    \frac{
    |\bar{Y}_k-\mathbb{E}_{p}[Y_i^k]|
    }{
    \bar{Z}_k
    }
    +
    |\mathbb{E}_{p}[Y_i^k]|
    \left|
    \frac{1}{\bar{Z}_k}
    -
    \frac{1}{\mathbb{E}_{p}[Z_i^k]}
    \right|. \nonumber\\
    &\le
    \frac{
    |\bar{Y}_k-\mathbb{E}_{p}[Y_i^k]|
    }{
    \bar{Z}_k
    }
    +
    \mathbb{E}_{p}[Z_i^k]
    \frac{
    |\bar{Z}_k-\mathbb{E}_{p}[Z_i^k]|
    }{
    \bar{Z}_k\mathbb{E}_{p}[Z_i^k]
    } \nonumber \\
    &\le
    \frac{2\epsilon_L}{a} + \frac{2\epsilon_L}{a} \nonumber\\
    &= \frac{4\epsilon_L}{a}.
\end{align}
The second inequality holds because $|\mathbb{E}_{p}[Y_i^k]|=\left|\int_{\mathcal{T}}x(t)\phi_k(t)p(t)dt\right| \le\int_{\mathcal{T}}\phi_k(t)p(t)dt=\mathbb{E}_{p}[Z_i^k]$.
Taking the maximum over all $k\in[K]$, we obtain
\begin{align}
\label{eq1}
    \max_{k\in[K]}
    |\hat{c}_k-c_k^p|
    \le
    \frac{4m}{a}
    \sqrt{\frac{2\log(4K/\delta)}{L}}.
\end{align}

Next, we bound the sampling bias term $|c_k^p-c_k|$. Define
\begin{align}
    A_k=\int_{\mathcal{T}}\phi_k(t)dt,
    \qquad
    B_k=\int_{\mathcal{T}}x(t)\phi_k(t)dt,
\end{align}
and
\begin{align}
    A_k^p=\int_{\mathcal{T}}\phi_k(t)p(t)dt,
    \qquad
    B_k^p=\int_{\mathcal{T}}x(t)\phi_k(t)p(t)dt.
\end{align}
Then
\begin{align}
    c_k=\frac{B_k}{A_k},
    \qquad
    c_k^p=\frac{B_k^p}{A_k^p}.
\end{align}
Thus,
\begin{align}
    |c_k^p-c_k|
    &=
    \left|
    \frac{B_k^p}{A_k^p}
    -
    \frac{B_k}{A_k}
    \right| \nonumber \\
    &\le
    \frac{|B_k^p-B_k|}{A_k^p}
    +
    |B_k|
    \left|
    \frac{1}{A_k^p}
    -
    \frac{1}{A_k}
    \right| \nonumber\\
    &\le 
    \frac{|B_k^p-B_k|}{A_k^p}
    +
    A_k
    \frac{|A_k^p-A_k|}{A_k^p A_k} \nonumber\\
    &\le
    \frac{|B_k^p-B_k| + |A_k^p-A_k| }{a} \nonumber \\
    &= 
    \frac{\left|
    \int_{\mathcal{T}}x(t)\phi_k(t)(p(t)-1)dt
    \right| + \left|
    \int_{\mathcal{T}}\phi_k(t)(p(t)-1)dt
    \right|}{a} \nonumber \\
    &\le 
    \frac{
    \int_{\mathcal{T}}|x(t)|\phi_k(t)|p(t)-1|dt +
    \int_{\mathcal{T}}\phi_k(t)|p(t)-1|dt}{a} \nonumber\\
    &\le
    \frac{2m \int_{\mathcal{T}}|p(t)-1|dt }{a} \nonumber\\
    &=
    \frac{2m \|p(t)-1\|_{L^1}}{a}
\end{align}
The second inequality holds due to $|B_k|=\left|\int_{\mathcal{T}}x(t)\phi_k(t)dt\right|\le\int_{\mathcal{T}}\phi_k(t)dt=A_k$.
Taking the maximum over $k\in[K]$, we obtain
\begin{align}
\label{eq2}
    \max_{k\in[K]}
    |c_k^p-c_k|
    \le
    \frac{2m}{a}\|p(t)-1\|_{L^1}.
\end{align}

Finally, combining the finite-sample approximation error in Eq.~\eqref{eq1} and the sampling bias term in Eq.~\eqref{eq2}  yields
\begin{align}
    \max_{k\in[K]}|\hat{c}_k-c_k|
    &\le
    \max_{k\in[K]}|\hat{c}_k-c_k^p|
    +
    \max_{k\in[K]}|c_k^p-c_k| \nonumber \\
    &\le
    \frac{4m}{a}
    \sqrt{\frac{2\log(4K/\delta)}{L}}
    +
    \frac{2m}{a}\|p(t)-1\|_{L^1}.
\end{align}
This completes the proof.

\section{Debiasing Property of Density-Corrected Estimation}
In this section, we further explain why the density-corrected estimator can eliminate the sampling bias. Recall when the observed timestamps are sampled from a non-uniform density $p(t)$, the uncorrected discrete estimator converges to the response coefficient under the timestamp sampling density, which leads to the non-vanishing bias term in Theorem~1.

To remove this bias, we apply importance weighting to compensate for the timestamp sampling density. The density-corrected estimator is defined as
\begin{align}
    \tilde c_k =
    \frac{
    \sum_{i=1}^{L}\frac{x_i\phi_k(t_i)}{p(t_i)}
    }{
    \sum_{i=1}^{L}\frac{\phi_k(t_i)}{p(t_i)}
    }.
\end{align}
The key difference from the uncorrected estimator is that the numerator and denominator now directly estimate the corresponding continuous-time integrals under the uniform time domain. Specifically, since $t_i$ is sampled from density $p(t)$, we have
\begin{align}
    \mathbb{E}_{p}\left[
    \frac{x(t)\phi_k(t)}{p(t)}
    \right]
    &=
    \int_{\mathcal{T}}
    \frac{x(t)\phi_k(t)}{p(t)}p(t)dt
    =
    \int_{\mathcal{T}}x(t)\phi_k(t)dt,
    \\
    \mathbb{E}_{p}\left[
    \frac{\phi_k(t)}{p(t)}
    \right]
    &=
    \int_{\mathcal{T}}
    \frac{\phi_k(t)}{p(t)}p(t)dt
    =
    \int_{\mathcal{T}}\phi_k(t)dt.
\end{align}
Therefore, the density-corrected estimator directly approximates the desired coefficient $c_k$, rather than the sampling-density-weighted coefficient $c_k^p$.

To obtain a finite-sample bound, we additionally assume that the sampling density is lower bounded, i.e., there exists a constant $\rho>0$ such that $p(t)\ge \rho$, for $\forall t\in\mathcal{T}$.
This condition ensures that the importance weights $1/p(t)$ are bounded. Under this assumption, together with $|x(t)|\le 1$ and $0\le \phi_k(t)\le m$, we have
\begin{align}
    \left|
    \frac{x(t_i)\phi_k(t_i)}{p(t_i)}
    \right|
    \le
    \frac{m}{\rho},
    \qquad
    0\le
    \frac{\phi_k(t_i)}{p(t_i)}
    \le
    \frac{m}{\rho}.
\end{align}
Following the same concentration and ratio-estimation arguments as in Appendix~A, when $L$ is sufficiently large, with probability at least $1-\delta$, the following bound holds:
\begin{align}
    \max_{k\in[K]}|\tilde c_k-c_k|
    \le
    \frac{4m}{a\rho}
    \sqrt{\frac{2\log(4K/\delta)}{L}}.
\end{align}

Compared with the bound for the uncorrected estimator in Theorem~1, the above bound does not contain the non-vanishing bias term related to $\|p(t)-1\|_{L^1}$. Therefore, the density-corrected estimator removes the systematic bias induced by irregular timestamp sampling.

\section{Datasets and Baseline Model Details}
\subsection{Datasets}
In this section, we provide a detailed description of the five public datasets used in our experiments, including their sources and characteristics. The basic statistics are summarized in Table~\ref{tab:dataset_statistics}.

\begin{table}[t]
\centering
\resizebox{\linewidth}{!}{
\begin{tabular}{lcccc}
\toprule
Dataset & \# Variables & \# samples & Avg \# Obs. & Max Length \\
\midrule
PhysioNet      & 36 & 11,981 & 308.6 & 47  \\
MIMIC          & 96 & 21,250 & 144.6 & 96  \\
HumanActivity  & 12 & 1,359  & 362.2 & 131 \\
StudentLife    & 9  & 20 & 7,680.5 & 1,370 \\
USHCN          & 5  & 1,114  & 313.5 & 337 \\
\bottomrule
\end{tabular}
}
\caption{Dataset statistics.}
\label{tab:dataset_statistics}
\end{table}

\begin{table*}[t]
\centering
\setlength{\tabcolsep}{5pt}
\renewcommand{\arraystretch}{1.35}
\normalsize
\caption{MAE results of DNBNet and baselines on five irregular time-series datasets. Best results are in bold and second-best results are underlined.}
\begin{tabular}{lccccc}
\toprule
\textbf{Method} & \textbf{USHCN} & \textbf{Human Activity} & \textbf{Student Life} & \textbf{PhysioNet} & \textbf{MIMIC} \\
\midrule
PrimeNet & $0.5094\pm0.0002$ & $1.7054\pm0.0005$ & $0.7231\pm0.0001$ & $0.6859\pm0.0000$ & $0.6632\pm0.0001$ \\
SeFT & $0.4981\pm0.0021$ & $0.9733\pm0.0006$ & $0.7201\pm0.0022$ & $0.6766\pm0.0005$ & $0.6689\pm0.0008$ \\
mTAN & $0.3412\pm0.0076$ & $0.2302\pm0.0071$ & $0.5742\pm0.0033$ & $0.4271\pm0.0041$ & $0.6749\pm0.0023$ \\
CRU & $0.3948\pm0.0069$ & $0.2673\pm0.0026$ & $0.6403\pm0.0043$ & $0.5782\pm0.0040$ & $0.5828\pm0.0046$ \\
GNeuralFlow & $0.3842\pm0.0024$ & $0.3802\pm0.0217$ & $0.6085\pm0.0163$ & $0.4012\pm0.0009$ & $0.5336\pm0.0058$ \\
Raindrop & $0.3515\pm0.0070$ & $0.2159\pm0.0027$ & $0.5816\pm0.0072$ & $0.4057\pm0.0006$ & $0.5047\pm0.0035$ \\
tPatchGNN & $0.4105\pm0.0725$ & $0.1444\pm0.0020$ & $0.5603\pm0.0009$ & $0.3665\pm0.0018$ & $0.4069\pm0.0015$ \\
GraFITi & $0.3437\pm0.0105$ & $0.1559\pm0.0056$ & $0.5644\pm0.0005$ & $\mathbf{0.3538\pm0.0015}$ & $\underline{0.4016\pm0.0040}$ \\
Warpformer & $\underline{0.3215\pm0.0017}$ & $0.1427\pm0.0005$ & $\mathbf{0.5510\pm0.0016}$ & $0.3718\pm0.0013$ & $0.4069\pm0.0014$ \\
Hi-Patch & $0.3364\pm0.0114$ & $0.1450\pm0.0008$ & $0.5580\pm0.0011$ & $0.3715\pm0.0036$ & $0.4112\pm0.0022$ \\
KAFNet & $0.3219\pm0.0028$ & $0.1501\pm0.0013$ & $0.5674\pm0.0011$ & $0.3809\pm0.0008$ & $0.4098\pm0.0022$ \\
APN & $0.3368\pm0.0147$ & $\underline{0.1401\pm0.0007}$ & $0.5682\pm0.0005$ & $0.3669\pm0.0002$ & $0.4145\pm0.0056$ \\
\midrule
\rowcolor{gray!8}
DNBNet (Ours) & $\mathbf{0.3053\pm0.0046}$ & $\mathbf{0.1376\pm0.0003}$ & $\underline{0.5515\pm0.0003}$ & $\underline{0.3644\pm0.0011}$ & $\mathbf{0.3981\pm0.0009}$ \\
\bottomrule
\end{tabular}
\label{tab:mae_results}
\end{table*}

\begin{table*}[t]
\centering
\setlength{\tabcolsep}{6pt}
\renewcommand{\arraystretch}{1.25}
\caption{Complete ablation study of DNBNet on five irregular time-series datasets. 
Best results are highlighted in \textbf{bold}.}
\begin{tabular*}{\textwidth}{@{\extracolsep{\fill}}lccccc}
\toprule
\multicolumn{6}{c}{\textbf{MSE}} \\
\midrule
\textbf{Variant} 
& \textbf{USHCN} 
& \textbf{Human Activity} 
& \textbf{Student Life} 
& \textbf{PhysioNet} 
& \textbf{MIMIC} \\
\midrule
w/o $p(t)$ 
& $0.4449\pm0.0291$
& $0.0564\pm0.0000$
& $0.6291\pm0.0004$
& $0.3161\pm0.0003$
& $0.4485\pm0.0012$ \\

Rep. RBF 
& $0.4144\pm0.0091$
& $0.0568\pm0.0002$
& $0.6242\pm0.0007$
& $0.3231\pm0.0007$
& $0.4507\pm0.0009$ \\

Rep. Fourier 
& $0.4101\pm0.0088$
& $0.0562\pm0.0001$
& $0.6257\pm0.0009$
& $0.3267\pm0.0002$
& $0.4869\pm0.0048$ \\

w/o AvgPool 
& $0.4195\pm0.0166$
& $0.0569\pm0.0001$
& $0.6269\pm0.0004$
& $0.3119\pm0.0006$
& $0.4494\pm0.0012$ \\

w/o BasDec 
& $0.4552\pm0.0424$
& $0.0559\pm0.0002$
& $0.6281\pm0.0011$
& $0.3092\pm0.0002$
& $0.4613\pm0.0023$ \\

DNBNet (Full) 
& $\mathbf{0.4013\pm0.0108}$
& $\mathbf{0.0551\pm0.0001}$
& $\mathbf{0.6136\pm0.0002}$
& $\mathbf{0.3089\pm0.0005}$
& $\mathbf{0.4289\pm0.0012}$ \\

\midrule
\multicolumn{6}{c}{\textbf{MAE}} \\
\midrule
\textbf{Variant} 
& \textbf{USHCN} 
& \textbf{Human Activity} 
& \textbf{Student Life} 
& \textbf{PhysioNet} 
& \textbf{MIMIC} \\
\midrule
w/o $p(t)$ 
& $0.3276\pm0.0200$
& $0.1386\pm0.0006$
& $0.5575\pm0.0005$
& $0.3694\pm0.0002$
& $0.4078\pm0.0014$ \\

Rep. RBF 
& $0.3079\pm0.0064$
& $0.1382\pm0.0003$
& $0.5549\pm0.0014$
& $0.3798\pm0.0012$
& $0.4116\pm0.0009$ \\

Rep. Fourier 
& $0.3139\pm0.0052$
& $0.1399\pm0.0005$
& $0.5571\pm0.0011$
& $0.3794\pm0.0001$
& $0.4239\pm0.0033$ \\

w/o AvgPool 
& $0.3147\pm0.0110$
& $0.1388\pm0.0003$
& $0.5569\pm0.0005$
& $0.3686\pm0.0010$
& $0.4080\pm0.0009$ \\

w/o BasDec 
& $0.3395\pm0.0278$
& $0.1387\pm0.0006$
& $0.5565\pm0.0011$
& $0.3652\pm0.0013$
& $0.4169\pm0.0022$ \\

DNBNet (Full) 
& $\mathbf{0.3053\pm0.0046}$
& $\mathbf{0.1376\pm0.0003}$
& $\mathbf{0.5515\pm0.0003}$
& $\mathbf{0.3644\pm0.0011}$
& $\mathbf{0.3981\pm0.0009}$ \\
\bottomrule
\end{tabular*}
\label{tab:complete_ablation}
\end{table*}

\paragraph{MIMIC.}
MIMIC is a large, freely accessible critical care database~\cite{johnson2016mimic}. It contains de-identified health records from patients who stayed in intensive care units (ICUs). The dataset is highly detailed, including vital signs, medications, and laboratory measurements. In our experiments, we use the clinical time series collected from the first 48 hours of each patient’s ICU stay. The resulting benchmark contains 21,250 samples with 96 variables.

\paragraph{PhysioNet.}
PhysioNet is another widely used clinical benchmark for irregular time series analysis~\cite{silva2012predicting}. It was originally released for predicting in-hospital mortality of ICU patients. Each record consists of multivariate measurements collected during the first 48 hours after admission. The processed benchmark used in our experiments contains 11,981 samples and 36 variables, such as serum glucose and heart rate.

\paragraph{Human Activity.}
Human Activity is a biomechanics-related dataset from the UCI Machine Learning Repository. It contains sensor measurements collected from five subjects performing various activities. The benchmark includes 1,359 samples and 12 variables representing 3D positions. The data are naturally irregular because the sensors record information at non-uniform time intervals.

\paragraph{StudentLife.}
StudentLife is a human behavior sensing dataset collected from college students using smartphones and wearable devices~\cite{nepal2024capturing}. In our setting, we use its numerical time series modality, which includes activity-related features such as biking, walking, and sleeping. The irregularity mainly arises from human scheduling, device usage patterns, and asynchronous sensor recording, making it a realistic benchmark for irregular multivariate time series forecasting.

\paragraph{USHCN.}
USHCN (United States Historical Climatology Network) is a climate dataset containing long-term meteorological observations from weather stations across the United States~\cite{menne2015long}. The processed benchmark contains 1,114 samples and 5 variables, such as daily maximum temperature, daily minimum temperature, and precipitation. Although the data are recorded daily, missing observations are common in practice, which makes this dataset suitable for irregular time series forecasting. Following previous work, we use a subset covering the period from 1996 to 2000.

\subsection{Baseline Model Details}

We compare our method with 12 representative baseline models, covering irregular time series forecasting, classification, and interpolation methods. These baselines include PrimeNet, SeFT, mTAN, CRU, GNeuralFlow, Raindrop, tPatchGNN, GraFITi, Warpformer, Hi-Patch, KAFNet, and APN.

\paragraph{PrimeNet}~\cite{chowdhury2023primenet} is a pretraining framework for irregular multivariate time series, which learns generalizable temporal representations from irregular observations and can be adapted to downstream forecasting tasks.

\paragraph{SeFT}~\cite{horn2020set} is a set-function-based model that treats irregular observations as an unordered set and uses permutation-invariant aggregation to learn representations without requiring time alignment.

\paragraph{mTAN}~\cite{DBLP:conf/iclr/ShuklaM21} employs multi-time attention to map irregular observations onto a set of reference time points, thereby obtaining a regularized latent representation for downstream prediction.

\paragraph{CRU}~\cite{schirmer2022modeling} is a continuous-time recurrent model that combines recurrent updates with latent uncertainty modeling, making it suitable for irregularly sampled observations with varying time gaps.

\paragraph{GNeuralFlow}~\cite{mercatali2024graph} extends neural flow modeling with graph structures to capture both continuous-time dynamics and inter-variable dependencies in irregular multivariate time series.

\paragraph{Raindrop}~\cite{DBLP:conf/iclr/ZhangZTZ22} is a graph attention model for IMTS, which propagates information across variables and timestamps to handle missingness and asynchronous observations.

\paragraph{tPatchGNN}~\cite{DBLP:conf/icml/ZhangYL0024} is a patch-based graph neural network that partitions irregular sequences into temporal patches and models temporal and variable dependencies jointly.

\paragraph{GraFITi}~\cite{yalavarthi2024grafiti} represents irregular time series as a bipartite graph over observations, timestamps, and variables, allowing the model to operate directly on sparse observations without dense alignment.

 \paragraph{Warpformer}~\cite{zhang2023warpformer} is a Transformer-based model that introduces a temporal warping mechanism for multi-scale modeling, aiming to better capture irregular and non-stationary temporal patterns.

\paragraph{Hi-Patch}~\cite{DBLP:conf/icml/LuoZ0025} adopts a hierarchical patching strategy to model irregular time series at multiple temporal granularities, which helps capture both local dynamics and long-range dependencies.

 \paragraph{KAFNet}~\cite{zhou2026revitalizing} is a recent IMTS forecasting model that uses kernel-based aggregation to encode irregular observations into compact latent representations for future prediction.

\paragraph{APN}~\cite{liu2026rethinking} is a recent basis-function-based forecasting method that models irregular observations through adaptive temporal aggregation, providing a strong and efficient baseline for IMTS forecasting.

\section{Additional Experiments Results}
\subsection{Complete MAE Results}
Due to space limitations, we report only the MSE results in the main paper and provide the complete MAE results in Table~\ref{tab:mae_results}. DNBNet achieves the best MAE on USHCN, Human Activity, and MIMIC, and obtains the second-best results on Student Life and PhysioNet, which is generally consistent with the MSE results. This demonstrates that DNBNet remains effective under different evaluation metrics and across diverse application domains.

\subsection{Varying Lookback Lengths and Forecast Horizons}
To further evaluate the robustness of DNBNet under different temporal settings, we conduct additional experiments by varying the lookback lengths and forecast horizons. As shown in Fig.~\ref{fig:lookback_windown}, DNBNet consistently achieves stable and competitive performance across different lookback windows. This indicates that the proposed model does not rely on a specific input length, and can effectively extract useful temporal patterns from both shorter and longer historical observations.

We also report the results under different forecast horizons in Table~\ref{tab:prediction_length}. The lookback lengths follow the settings in Table~\ref{tab:dataset_statistics} of main paper, while the prediction horizons are set to the remaining length of each sequence, including 12 hours for MIMIC-III and PhysioNet, 300 milliseconds for Human Activity, the next 3 observations for USHCN, and 10 days for Student Life. The results show that DNBNet maintains strong performance under these horizon settings, further demonstrating the effectiveness of the debiased neural basis-function response and the dual-branch forecasting decoder for irregular time-series forecasting.

\begin{table*}[t]
\centering
\renewcommand{\arraystretch}{1.25}
\setlength{\tabcolsep}{6pt}
\caption{Experimental results on five irregular multivariate time-series datasets evaluated by MSE and MAE. The look-back lengths follow Table~2, and the prediction horizons are set to the remaining length of each sequence: 12 hours for MIMIC-III and PhysioNet, 300 milliseconds for Human Activity, the next 3 observations for USHCN, and 10 days for Student Life.}
\begin{tabular*}{\textwidth}{@{\extracolsep{\fill}}lccccc}
\toprule
\multicolumn{6}{c}{\textbf{MSE}} \\
\midrule
\textbf{Method} 
& \textbf{USHCN} 
& \textbf{Human Activity} 
& \textbf{Student Life} 
& \textbf{PhysioNet} 
& \textbf{MIMIC} \\
\midrule
GraFTi
& $0.2089{\pm}0.0218$
& $0.0431{\pm}0.0006$
& $0.5789{\pm}0.0018$
& $\mathbf{0.3553{\pm}0.0016}$
& $0.4957{\pm}0.0032$ \\

Warpformer
& $0.1738{\pm}0.0069$
& $0.0446{\pm}0.0013$
& $0.5861{\pm}0.0032$
& $0.3558{\pm}0.0016$
& $0.5295{\pm}0.0035$ \\

Hi-Patch
& $0.2416{\pm}0.0205$
& $0.0472{\pm}0.0001$
& $0.5902{\pm}0.0020$
& $0.3666{\pm}0.0001$
& $0.5352{\pm}0.0062$ \\

KAFNet
& $0.2045{\pm}0.0091$
& $0.0677{\pm}0.0006$
& $0.5980{\pm}0.0024$
& $0.3647{\pm}0.0010$
& $0.5145{\pm}0.0076$ \\

APN
& $0.1816{\pm}0.0133$
& $0.0436{\pm}0.0004$
& $0.6057{\pm}0.0140$
& $0.3689{\pm}0.0026$
& $0.5396{\pm}0.0053$ \\

\textbf{NBFN (Ours)}
& $\mathbf{0.1730{\pm}0.0097}$
& $\mathbf{0.0429{\pm}0.0001}$
& $\mathbf{0.5763{\pm}0.0037}$
& $0.3584{\pm}0.0006$
& $\mathbf{0.4902{\pm}0.0085}$ \\

\midrule
\multicolumn{6}{c}{\textbf{MAE}} \\
\midrule
\textbf{Method} 
& \textbf{USHCN} 
& \textbf{Human Activity} 
& \textbf{Student Life} 
& \textbf{PhysioNet} 
& \textbf{MIMIC} \\
\midrule
GraFTi
& $0.2813{\pm}0.0031$
& $0.1203{\pm}0.0030$
& $0.5308{\pm}0.0012$
& $\mathbf{0.4051{\pm}0.0003}$
& $0.4369{\pm}0.0023$ \\

Warpformer
& $0.2682{\pm}0.0026$
& $0.1221{\pm}0.0022$
& $0.5317{\pm}0.0015$
& $0.4063{\pm}0.0019$
& $0.4412{\pm}0.0023$ \\

Hi-Patch
& $0.2923{\pm}0.0105$
& $0.1250{\pm}0.0001$
& $0.5318{\pm}0.0011$
& $0.4150{\pm}0.0011$
& $0.4544{\pm}0.0062$ \\

KAFNet
& $0.2714{\pm}0.0135$
& $0.1582{\pm}0.0015$
& $0.5386{\pm}0.0016$
& $0.4138{\pm}0.0014$
& $0.4419{\pm}0.0046$ \\

APN
& $0.2818{\pm}0.0098$
& $0.1211{\pm}0.0011$
& $0.5480{\pm}0.0156$
& $0.4153{\pm}0.0027$
& $0.4564{\pm}0.0044$ \\

\textbf{NBFN (Ours)}
& $\mathbf{0.2518{\pm}0.0073}$
& $\mathbf{0.1172{\pm}0.0006}$
& $\mathbf{0.5291{\pm}0.0022}$
& $0.4084{\pm}0.0012$
& $\mathbf{0.4338{\pm}0.0067}$ \\
\bottomrule
\end{tabular*}
\label{tab:prediction_length}
\end{table*}

\begin{figure*}
    \centering
    \includegraphics[width=1\linewidth]{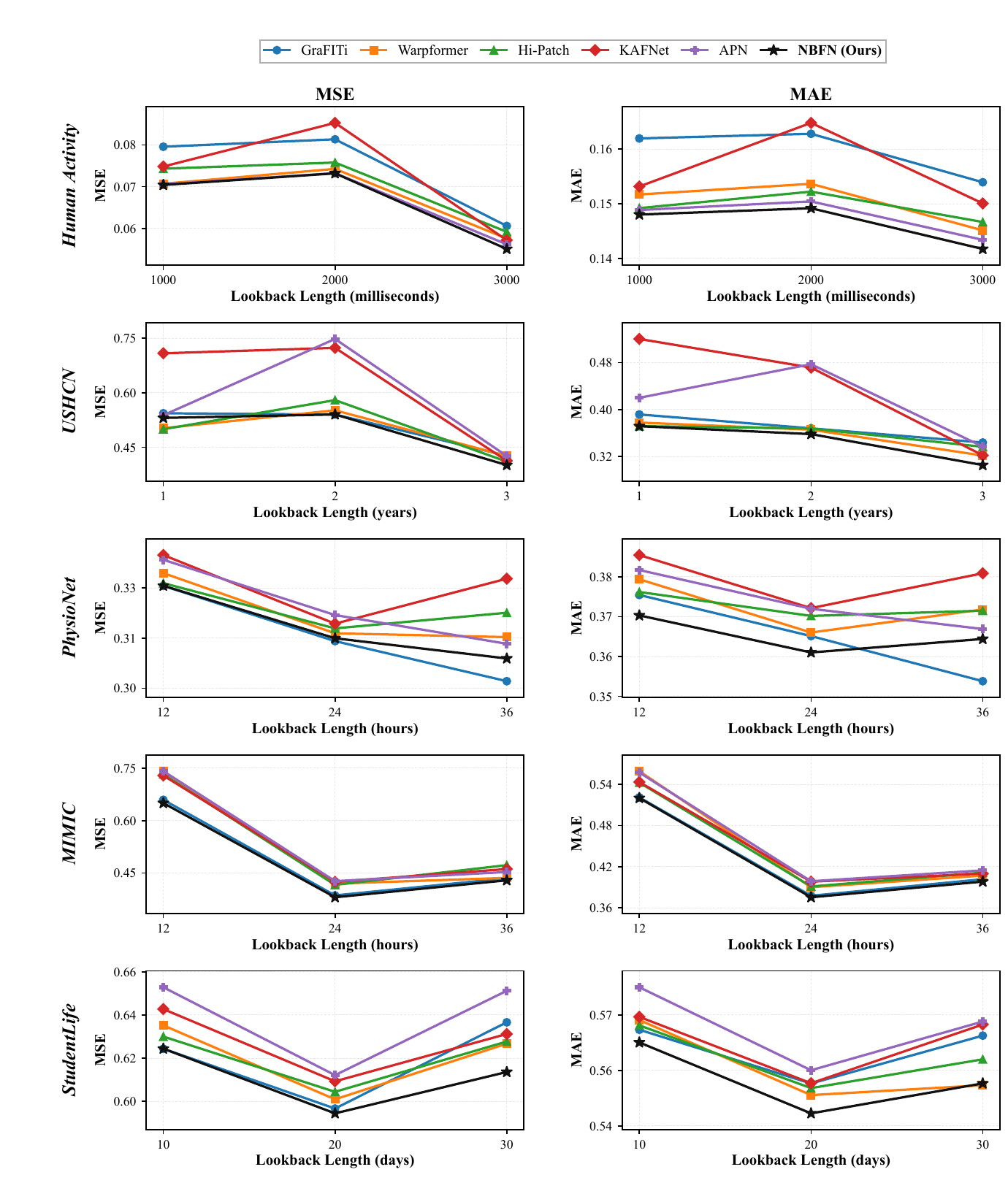}
    \caption{Forecasting performance with varying lookback lengths and fixed forecast horizons.}
    \label{fig:lookback_windown}
\end{figure*}

\subsection{Complete Ablation Results}
We further report the complete ablation results on all five datasets in Tables~\ref{tab:complete_ablation}. The full DNBNet consistently outperforms its variants, verifying the contribution of each component. In particular, removing the density correction term $p(t)$, replacing learnable basis functions with predefined RBF or Fourier bases, or removing the average pooling module or basis-function decoder all lead to performance degradation. These results confirm the effectiveness of debiased response estimation, learnable basis functions, multi-scale representation, and basis-function-based prediction.


\end{document}